\documentclass[twoside,10pt]{article}
\usepackage{ISG_style}

\usepackage[utf8]{inputenc} 
\usepackage[T1]{fontenc}    
\usepackage{hyperref}       
\usepackage[table]{xcolor}         

\definecolor{ballblue}{rgb}{0.13, 0.67, 0.8}
\definecolor{babypink}{rgb}{0.96, 0.76, 0.76}
\definecolor{antiquefuchsia}{rgb}{0.57, 0.36, 0.51}
\definecolor{blue(pigment)}{rgb}{0.2, 0.2, 0.6}
\definecolor{blush}{rgb}{0.87, 0.36, 0.51}
\hypersetup{
    colorlinks = true,
    citecolor = {red},
    linkcolor = {blue(pigment)}
}
\usepackage{url}            
\usepackage{booktabs}       
\usepackage{amsfonts}       
\usepackage{nicefrac}       
\usepackage{microtype}      
\usepackage{tabularx}
  
\usepackage{amssymb}
\usepackage{amsfonts}
\usepackage{mathrsfs}
\usepackage{dsfont}
\usepackage{graphicx}
\usepackage{upgreek}
\usepackage[ruled,vlined]{algorithm2e}
\usepackage{subcaption}
\usepackage{comment}
\usepackage{bigints}
\usepackage[title]{appendix}
\usepackage[numbers,compress]{natbib}

\DeclareMathOperator*{\argmax}{arg\,max}
\DeclareMathOperator*{\argmin}{arg\,min}

\newtheorem{theorem}{Theorem}
\newtheorem{assumption}{Assumption}

\newtheorem{lemma}{Lemma}

\newcommand{\rev}[1]{\textcolor{black}{#1}}

\makeatletter
\renewcommand{\paragraph}{%
  \@startsection{paragraph}{4}%
  {\z@}{.4ex \@plus 0ex \@minus .2ex}{-1em}%
  {\normalfont\normalsize\bfseries}%
}
\makeatother

\title{\bf \Large Combinatorial Multi-armed Bandits: \\
Arm Selection via Group Testing}

\author{Arpan Mukherjee \footnote{Department of Electrical, Computer, and Systems Engineering, 
      Rensselaer Polytechnic Institute.}
      \and
       Shashanka Ubaru \footnote{IBM T.J. Watson Research Center.}
      \and Keerthiram Murugesan \footnotemark[2] 
      \and Karthikeyan Shanmugam \footnotemark[2]  \and 
     Ali Tajer \footnotemark[1]
      }

\date{}

\begin{document}

\maketitle

\begin{abstract}
This paper considers the problem of combinatorial multi-armed bandits with semi-bandit feedback and a cardinality constraint on the super-arm size. Existing algorithms for solving this problem typically involve two key sub-routines: (1)~a {\em parameter estimation} routine that sequentially estimates a set of base-arm parameters, and (2)~a {\em super-arm selection} policy for selecting a subset of base arms deemed optimal based on these parameters. State-of-the-art algorithms assume access to an {\em exact} oracle for super-arm selection with unbounded computational power. At each instance, this oracle evaluates a list of score functions, the number of which grows as low as linearly and as high as exponentially with the number of arms. This can be prohibitive in the regime of a large number of arms. This paper introduces a novel realistic alternative to the perfect oracle. This algorithm uses a combination of {\em group-testing} for selecting the super arms and \emph{quantized} Thompson sampling for parameter estimation. Under a general separability assumption on the reward function, the proposed algorithm reduces the complexity of the super-arm-selection oracle to be \emph{logarithmic} in the number of base arms while achieving the same regret order as the state-of-the-art algorithms that use exact oracles. This translates to  \emph{at least an exponential} reduction in complexity compared to the oracle-based approaches.
\end{abstract}



\section{Introduction}
\vspace{-.05 in}
The combinatorial multi-armed bandit (CMAB) problem is a generalization of the stochastic multi-armed bandit problem, in which there is a set of {\em base} arms and a learner selects a {\em subset} of them at each round. Such sets of base arms are called a {\em super-arms}, and the set of all possible super-arms constitutes the action space of the learner~\citep{chen2013combinatorial, combes2015combinatorial,chen2016combinatorial,chen2016combinatorial1}.

\paragraph{Bandit versus semi-bandit feedback.} CMABs can be broadly divided into two settings according to the level of feedback a learner receives in response to its actions: the {\em bandit} and the {\em semi-bandit} feedback settings. In the bandit feedback setting, the learner pulls a super-arm and observes one aggregate reward value generated by the selected super-arm~\cite {nie2022explore,jia2019towards}. On the other hand, in the semi-bandit feedback setting, in addition to the aggregate reward, the learner has access to a set of stochastic observations generated by the individual arms that constitute the selected super arm ~\citep{chen2016combinatorial,wang2022thompson}. This paper focuses on semi-bandit feedback, and its objective is to minimize the average cumulative regret in CMABs under this feedback model. The CMAB model is assumed to belong to the class of Bernoulli bandits.

\paragraph{UCB versus Thompson sampling.} Regret minimization algorithms for CMABs with semi-bandit feedback consist of two key sub-routines: an estimation routine and a super-arm selection routine. The estimation routine aims to form reliable estimates of the unknown parameters of the base arms. The super-arm selection routine specifies the sequential selection of the super-arms over time. Super-arm selections rely on the estimates formed by the estimation routine, and there is a wide range of arm-selection rules based on the upper confidence bound (UCB) principle~\citep{chen2013combinatorial,kveton2015tight,combes2015combinatorial,chen2016combinatorial} or Thompson sampling (TS)~\citep{wang2022thompson,perrault2021statistical}. Recent studies demonstrate that the TS-based approaches are more efficient and empirically outperform the UCB-based counterparts. Specifically, the combinatorial Thompson sampling (CTS) algorithm in~\citep{wang2022thompson} adopts a posterior sampling estimator for the bandit mean values, and uses an oracle that perfectly determines the set of super-arms that are optimal for the estimated means. Under such access to an {\em exact} oracle, the studies  in~\citep{wang2022thompson}~and~\citep{perrault2021statistical} establish that the CTS algorithm achieves an order-wise optimal regret of $O\big(\frac{m}{\Delta}\log T\big)$, where $m$ denotes the number of base arms, $T$ is the horizon, and $\Delta$ specifies the minimum expected reward gap between an optimal super-arm and any other non-optimal super-arm.

\paragraph{Oracle complexity.} Accessing an exact oracle is often computationally prohibitive. 
In this paper, our objective is to alleviate the {\em oracle complexity} of existing methods. This is motivated by the fact that black-box function evaluations can be expensive, and hence, it is imperative to minimize the number of black-box queries to the oracle. It is noteworthy that there exist \emph{approximate} alternatives to the exact oracle, which require a polynomial complexity in the number of base arms. While offering an improvement in complexity, polynomial complexity can still be excessive, and more importantly, approximate solutions can result in \emph{linear} cumulative regret. 
Examples of reward functions facing such issues include submodular reward functions~\citep{krause2014submodular} and reward functions modeled as the output of a neural network~\citep{hwang2023combinatorial}. Therefore, to avoid linear regret, CTS has to inevitably rely on an exact oracle, the computational complexity of which, in general, grows exponentially with the number of base arms.

\paragraph{Group testing.} Group testing (GT) is an efficient approach for solving large-scale combinatorial search problems~\citep{dorfman1943detection,du2000combinatorial}. The basic premise in GT is that a small sub-population (size $K$) of a large body (size $m\gg K$) has a certain property (e.g., being defective), and the objective of group testing is to identify them without individually testing all. To avoid individual tests, the population members are \emph{pooled} into groups, and the group is tested as a whole. The majority of tests are expected to return negative results, i.e., most groups do not have a member with the desired property. This clears the entire group, significantly saving the number of tests administered.

The number of tests required for identifying the defective items broadly varies depending on the settings (we refer to~\citep{CIT-099} for a review). Under both noiseless as well as noisy test outcomes (with a bound on the number of noisy measurements), when $K=O(m^\alpha)$ where $\alpha\in(0,1/3)$, only $O(K^2\log m)$ tests are sufficient to recover the defective subset perfectly (zero-error criterion)~\citep{hwang1987non,du2000combinatorial,chan2011non}. Under a vanishing error criterion, the number of tests can be reduced to  $O(K\log m)$~\citep{zhigljavsky2003probabilistic,gilbert2012recovering} (partial recovery). Furthermore, GT schemes can be classified into adaptive and non-adaptive methods. In non-adaptive group testing, all the tests are decided at once. In contrast, in adaptive group testing, the tests are divided into stages, and the tests for a particular stage are decided based on the outcomes of the previous stage. Adaptive group testing has been shown to significantly reduce the number of tests, requiring only $O(K\log m + m)$ tests for exact recovery~\citep{hwang1975generalized,de2005optimal}.

Different variants of group testing have also been proposed in the literature~\citep{du2000combinatorial,du2006pooling,d2014lectures}. These include threshold group testing~\citep{damaschke2006threshold}, where a test result is positive if the number of defective items in the pool is above a threshold; quantitative or additive group testing~\citep{d2014lectures,du2000combinatorial}, where the test output is the number of defective items in the pool; probabilistic group testing~\citep{cheraghchi2011group}, where we wish to recover the defective items with high probability; graph-constrained group testing~\citep{sihag2021adaptive}, where there are constraints how items can be grouped; and semi-quantitative group testing~\citep{emad2014semiquantitative,cheraghchi2021semiquantitative}, where the (additive) test outputs are quantized into a fixed set of thresholds. GT has been also adopted to solve large-scale learning problems, such as feature selection~\citep{zhou2014parallel}, extreme classification~\citep{ubaru2017multilabel,ubaru2020multilabel}, and data valuation~\citep{jia2019towards}. 

\paragraph{Contributions.} In this paper, we leverage GT to dispense with the assumption of exact oracle access for the CTS algorithm. This results in an \emph{exponential reduction} in the oracle complexity without compromising the achievable regret. Specifically, we devise the 
called  {\bf G}roup {\bf T}esting + {\bf Q}uantized {\bf T}hompson {\bf S}ampling (GT+QTS) algorithm, which 
under a \rev{mild probabilistic assumption on the separability} of the reward function (Assumption~\ref{assumption: C separability}), 
will have exponentially lower complexity compared to an exact oracle. GT+QTS has two key innovations compared to the existing algorithms. First, the complexity reduction is enabled by GT, the success of which fundamentally relies on separability assumptions, lacking which we may face sub-optimal (linear) regret. To address this, as a second contribution, we devise a {\em quantization} scheme that ensures the probabilistic separability of the reward function. We show that the GT-based oracle requires only $O(\log m)$ 
black-box queries to discern the optimal set of arms in each round. Furthermore, we show that the GT+QTS algorithm preserves the optimal regret order of $O\left(\frac{m}{\Delta}\log T\right)$ while providing an exponential reduction in the oracle complexity.

\paragraph{Related works.}
We provide an overview of the most closely related studies to the scope of this paper.  
The theoretical analysis of the TS-based approaches for MABs was first provided in~\citep{kaufmann2012thompson,agrawal2012analysis}. These results were later improved in~\citep{agrawal2013further} and extended to a general action space and feedback in~\citep{gopalan2014thompson}. The CMAB problem is studied under different settings in~\citep{chen2013combinatorial,combes2015combinatorial,chen2016combinatorial1,chen2016combinatorial}.
The TS-based approach to CMAB is investigated for top-$K$ CMAB in~\citep{komiyama2015optimal}, analyzed for contextual CMAB in~\citep{wen2015efficient}, and studied under Bayesian regret metric by~\cite{russo2016information}. Furthermore, CMAB has been investigated in the full-bandit feedback setting in~\cite{nie2022explore}.

The study closest to the scope of this paper is~\cite{wang2022thompson}, which analyzes the CTS algorithm to solve combinatorial semi-bandits under a Bernoulli model and a Beta prior distribution for the belief parameters. It establishes that the CTS algorithm asymptotically achieves the optimal regret
Another related study is by~\cite{perrault2021statistical}, which presents a tighter regret bound for the Beta prior, and a similar optimal regret analysis is established for multivariate sub-Gaussian outcomes using Gaussian priors.


\section{Combinatorial Bandits}
\label{sec:combbandit}

\textbf{Setting.} Similarly to the canonical models in~\citep{wang2022thompson,perrault2021statistical}, we consider a CMAB setting with $m$ arms, and define the set $[m]:=\{1,\cdots,m\}$. Each arm $i\in[m]$ is associated with an independent Bernoulli distribution with an \emph{unknown} mean $\mu_i$. We denote the vector of {\em unknown} mean values by $\bmu:=[\mu_1,\cdots,\mu_m]$. Sequentially over time, the learner selects subsets of arms, which we refer to as \emph{super-arms}. The super-arm selected at time $t$ is denoted by $\mcS(t)\in \mcI$, where $\mcI\subset 2^{[m]}$ specifies the set of {\em permissible} super-arms. We consider the semi-bandit feedback model, wherein, at each time $t$, upon pulling a super-arm $\mcS(t)$, the learner observes a feedback 
\begin{align}
    Q(t):=\{X_i(t):i\in\mcS(t)\}\ ,
\end{align}
where $X_i(t)$ denotes a random observation from arm $i\in[m]$, i.e., $X_i(t)\sim\rm{Bern}(\mu_i)$. In addition to the feedback $Q(t)$, based on the super-arm selected at time $t$, the learner gains a reward  $R(t)$.
The average reward $\E[R(t)]$ is assumed to depend only on the mean values of the arms $i\in\mcS(t)$. To formalize this, we assume that there exists a function $r:\mcI\times[0,1]^m\mapsto\R$, such that 
\begin{align}
\E[R(t)] = r(\mcS(t) \;;\; \bmu  )\ ,    
\end{align}
where the expectation is with respect to the measure induced by the distributions of arms $i\in\mcS(t)$. Function $r$ is assumed to be \emph{unknown}, and the learner only has {\em black-box} access to it, i.e., for any $\btheta\in[0,1]^m$ and $\mcS\in\mcI$, the learner queries the black-box and obtains the reward evaluation $r(\mcS\;;\;\btheta)$. For any $\btheta\in[0,1]^m$, we define the optimal super-arm associated with $\btheta$ as the permissible set with the largest reward, i.e., 
\begin{align}
\label{eq:S_star}
    \mcS^\star(\btheta)\;:=\;\argmax\limits_{\mcS\in\mcI} \; r(\mcS\;;\;\btheta)\ .
\end{align}
If there are multiple optimal super-arms, we randomly select one of them. We also assume that the cardinality of the optimal set $\mcS^\star(\bmu) = K$.
For a given $\btheta$ and any set $\mcS\in\mcI$, we define the sub-optimality with respect to $\mcS^\star(\btheta)$ by 
\begin{align}
    \Delta(\mcS,\btheta)\;:=\;r(\mcS^\star(\btheta)\;;\;\btheta) - r(\mcS\;;\;\btheta)\ .
\end{align}
Accordingly, we define the minimal and maximal sub-optimality gaps for any parameter $\btheta\in[0,1]^m$ as 
\begin{align}
    \Delta_{\min}(\btheta)\; &:=\; \min\limits_{\mcS\in\mcI : \Delta(\mcS,\btheta)>0}\; \Delta(\mcS,\btheta)\ ,\\
    \Delta_{\max}(\btheta)\; &:=\; \max\limits_{\mcS\in\mcI}\; \Delta(\mcS,\btheta)\ .
\end{align}
The learner's objective is to minimize the {\em average} cumulative regret $\mathfrak{R}(T)$, which is defined as
\begin{align}
\label{eq:regret}
    \mathfrak{R}(T)\;:=\;\sum\limits_{t=1}^T\E[\Delta(\mcS(t),\bmu)]\ ,
\end{align}
where the expectation is taken with respect to the measure induced by the interaction of the learner with the bandit instance. For any set $\mcS\subseteq [m]$ and $\btheta\in[0,1]^m$, we define $\btheta_{\mcS}$  as the vector, whose entries are equal to $\btheta$ for every $i\in\mcS$, and $0$ otherwise.

\paragraph{Assumptions.} We start by discussing some of the commonly used assumptions in the CMAB literature on the reward function $r$~\citep{wang2022thompson,perrault2021statistical}. Then, we will discuss how to relax some of the idealized assumptions in the literature. Specifically, the existing studies relevant to this work assume access to an exact oracle that can perfectly solve the problem in~\eqref{eq:S_star}, i.e., identifies the optimal super-arm $\mcS^\star(\btheta)$ for any parameter $\btheta\in[0,1]^m$. In this paper, we relax this assumption and replace the oracle with a procedure with only soft (probabilistic) guarantees for solving~\eqref{eq:S_star}.
We start with the following common assumption in the CTS-based approaches for CMAB; see~\citep {wang2022thompson,perrault2021statistical}. 
\begin{assumption}
\label{assumption: mean only}
The expected reward of a super-arm $\mcS \in \mcI$ depends only on the mean values of the base arms in $\mcS$. 
\end{assumption}
We note that some studies on the confidence interval-based methods have relaxed this assumption~\citep{chen2016combinatorial}. In the context of CTS, relaxing this assumption poses a few technical challenges. Specifically, a TS-based approach at each step samples the super-arm that maximizes the reward function based on posterior {\em mean} estimates. However, for rewards, which depend on the arm distributions (and not just the mean values), we need estimates for the distributions. This calls for a separate algorithm design.  Our next assumption quantifies the smoothness of the reward function.
\begin{assumption}[Lipschitz continuity]
\label{assumption:Lipschitz}
  The reward function is globally $B$-Lipschitz in $\btheta$. More specifically, for any $\mcS\in\mcI$ and for any $\btheta,\btheta^\prime\in[0,1]^m$,  the reward function satisfies $|r(\mcS\;;\;\btheta) - r(\mcS\;;\;\btheta^\prime)| \leq B||\btheta_{\mcS} -\btheta^\prime_{\mcS}||_1$ for some universal constant $B\in\R_+$.
\end{assumption}
Next, we specify assumptions on the variations of the reward function with respect to $\mcS$. We adopt a common monotonicity assumption based on which adding arms to any super-arm will not decrease the reward. 
\begin{assumption}[Reward monotonicity]
\label{assumption: monotone}
    The reward function $r$ is monotone and increasing in $\mcS$, i.e., for any $\mcS_1\subseteq\mcS_2\subseteq [m]$  we have $r(\mcS_1\;;\;\btheta) \leq r(\mcS_2\;;\;\btheta), \forall \btheta\in[0,1]^m$.
\end{assumption}
Without any constraint on the cardinality of the optimal set, the monotonicity assumption implies that the optimal solution super-arm is $[m]$. To avoid this, we impose that the cardinality of the optimal super-arm $|\mcS| \leq K\in[m]$.
Besides the above standard CMAB assumptions, we also adopt three more assumptions pertinent to dispensing with access to the exact oracle that solves~\eqref{eq:S_star} and designing an efficient probabilistic alternative. The following two assumptions are needed for determining the number of tests in our GT procedure.
\begin{assumption}[Bounded reward]
\label{assumption: bounded reward}
    For any set $\mcS\in\mcI$ and any $\btheta\in[0,1]^m$, we assume that the reward function satisfies $r(\mcS\;;\;\btheta)\in[0,M]$, where $M$ is known.
\end{assumption}
\rev{Next, we introduce a probabilistic assumption on the distribution of the minimum gap of the bandit instances. This assumption is critical in order to facilitate an exponential reduction in the oracle complexity. Furthermore, this assumption covers the case where a lower bound on the minimum gap of the class of instances is known/assumed to be known, which is a common occurrence in many applications such as principal component analysis (minimum singular value gap requirement for iterative partial SVD algorithms~\cite{musco2015randomized}), topological data analysis (minimum gap requirement for Betti number estimation~\cite{apers2023simple}), and others.
\begin{assumption}
\label{assumption:Delta_min}
     The probability of distribution of the minimum gap $\Delta_{\min}(\bmu)$ is known, and its cumulative distribution function (CDF) is denoted by $\mathbb{F}_{\bmu}$.
\end{assumption}}
The next assumption states that augmenting any subset of arms with an optimal arm results in higher reward gain than augmenting with a non-optimal arm. 
\begin{assumption}[Separable reward]
\label{assumption: C separability}
    For any parameter $\btheta\in[0,1]^m$, any optimal arm $s\in\mcS^\star(\btheta)$, any sub-optimal arm $\tilde{s}\notin \mcS^\star(\btheta)$, and any set $\mcS\subset [m]\setminus\{s,\tilde{s}\}$, we have\footnote{In the case of multiple optimal super-arms, $s$ belongs to the union of optimal arms, and $\tilde{s}$ does not belong to it.}:
    \begin{align}
        r(\mcS\cup\{s\}\;;\;\btheta) - r(\mcS\cup\{\tilde{s}\}\;;\;\btheta) > 0\ .
    \end{align}
\end{assumption}
Several commonly used set-valued functions naturally satisfy the separability assumption, e.g., linear rewards, i.e., $r(\mcS\;;\;\btheta) = \sum_{i\in\mcS} \theta_i$, information measure-based functions such as mutual information and $f$-divergence~\citep{zhou2014parallel,nguyen2010estimating}. Furthermore, we show that a two-layer neural network (NN) also satisfies the separability assumption (see Theorem~\ref{theorem: ANN}, Appendix~\ref{proof: ANN}).
\section{Algorithm: GT + Quantized TS }
\label{sec:algo}

In this section, we provide the details of the GT+QTS algorithm, the objective of which is minimizing the average cumulative regret defined in~\eqref{eq:regret}. 
This algorithm has two central sub-routines. The first sub-routine is an estimation process that computes estimates for the base arm means. The second sub-routine is a procedure that sequentially, at each round, determines an optimal super-arm to be pulled based on the current base arms mean estimates. These procedures are discussed next.

\subsection{TS-based Estimator}
We consider a TS-based approach, where the estimates of the mean values are generated by sampling from a posterior distribution. We adopt a beta distribution  to generate the posteriors. A beta posterior naturally comes up as the conjugate distribution assuming uniform priors on the mean values of the base arms. We denote the distribution associated with arm $i\in[m]$ at time $t$ by ${\sf Beta}(a_i(t),b_i(t))$. We initialize, for $t=0$, $a_i(0)=b_i(0)=1$ for all arms, in which case the beta distribution reduces to a uniform distribution. 
Subsequently, for each time $t\in\N$, a super-arm $\mcS(t)$ is selected, and we receive the feedback $Q(t)$. Based on the feedback, we update the prior distribution of each base arm by updating $a_i(t)$ and $b_i(t)$. Furthermore, recall that $X_i(t)$ denotes the feedback from the base arm $i\in\mcS(t)$. We draw a sample $Y_i(t)\sim {\sf Bern}(X_i(t))$, and update the posterior distribution as follows.
\begin{align}
\label{eq: posterior1}
    a_i(t+1)\; &=\; a_i(t) + Y_i(t)\ ,\\
\label{eq: posterior2}
    b_i(t+1) \; &= \; b_i(t) - Y_i(t)+1 \ .
\end{align}
Finally, our estimate for $\bmu$ at time $t$ is a random sample from the beta distribution with parameters specified in~\eqref{eq: posterior1}-\eqref{eq: posterior2}, i.e., we generate the posterior estimate $\btheta(t+1)$ according to $\theta_i(t+1)\sim {\sf Beta}(a_i(t+1),b_i(t+1))$.

\subsection{GT-based Arm Selection}
We design a GT-based procedure to select the optimal super-arm in each round. The nature of this procedure is probabilistic, and  it is designed to find the optimal super-arm with a high probability. 

GT involves pooling together several arms and performing a test on the pooled set. Tests are repeated by selecting and pooling different subsets of arms for each test. When we have $\ell$ tests, the pooling process can be characterized by a test matrix $\bA\in\{0,1\}^{\ell\times m}$, where row $j\in[\ell]$ specifies the arms that are included in test $j$. Specifically, $A_{j,i}=1$ if the arm $i\in[m]$ is contained in test $j\in[\ell]$ test, and otherwise $A_{j,i}=0$. For each test $j\in[\ell]$, we design a function $\rho_j : 2^{[m]}\times [0,1]^m \mapsto [0,M]$, that assigns score to the outcome of test $j$. Next, based on these test scores, we assign a grade to each arm that specifies whether the arm is likely to be in the optimal super-arm or not. This grade assignment is formalized by a decoding mechanism specified by the function $\phi_i : [0,M]^{\ell} \mapsto \R$, which generates the arms' grades. Subsequently, a candidate super-arm is selected as the set of arms with the top $K$  grades.

\begin{figure}[t]
    \centering
        \includegraphics[width=0.35\textwidth,trim={1cm 0cm  1cm  0cm },clip]{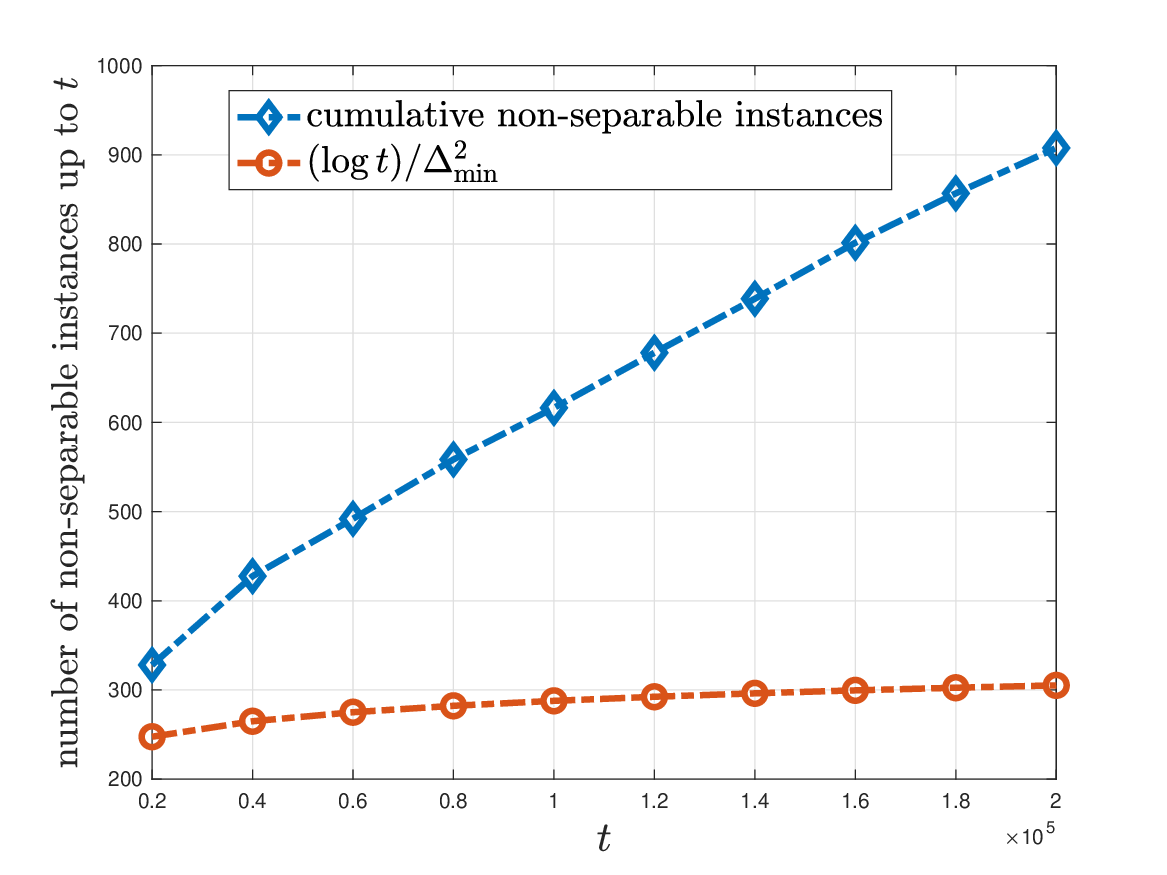}
        \caption{Cumulative number of times that the instance $\btheta(t)$ is non-separable.}
        \label{fig:1}
\end{figure}

\paragraph{Group-testing oracle (GTO).} Next, we describe our GT encoding and decoding mechanisms. To lay context, we first describe a naïve adaptation of the GT approach in~\citep{zhou2014parallel}. It was designed for ranking, and can be used to find the optimal super-arm at each step. We then describe a shortcoming of this naïve approach and modify it to replace the exact oracle used by CTS.

\textit{Naïve GT approach.} We adopt a randomized testing mechanism, in which each arm $j\in[m]$ is included in the test by flipping a coin. Specifically, arm $i\in[m]$ is included in test $j\in[\ell]$ based on a Bernoulli random variable $A_{j,i}\sim{\sf Bern}(p)$ such that arm $i$ is included in the test if $A_{j,i}=1$. Probability $p$ is a design parameter to be chosen later. For designing the decoder, at each time $t\in\N$, we set the scoring function of test $j$, i.e., $\rho_j$, to be an evaluation of the average reward function at the current estimates of the mean values, i.e., $\rho_j(t) := r(\bA_j\;;\;\btheta(t))$, where $\bA_j$ denotes the $j^{\sf th}$ row of the test matrix $\bA$. Based on the test scores $\boldsymbol{\rho}(t) := (\rho_1(t),\cdots,\rho_{\ell}(t))$, we define the arm grading function $\bphi(t) := (\phi_1(t),\cdots,\phi_m(t))$ as follows.
\begin{align}
\label{eq:decoder}
    \bphi(t)\; = \; \bA^{\top}\brho(t)\ .
\end{align}
For each arm $i\in[m]$, the GT decoder in~\eqref{eq:decoder} considers the tests $j\in[\ell]$ which contain $i$, and adds up the scores due to these tests to form an aggregate grade for each arm $i$. If an arm $i\in[m]$ is contained in multiple tests with high scores, and the resulting aggregate score is large, it is highly likely that the arm $i$ is responsible for the high scores assigned to the tests. Hence, arm $i$ is a more likely candidate to be one of the arms in the optimal super-arm. Hence, the arms with the top-$K$ grades are selected as candidates for the optimal super-arm to be pulled at time $t$.

\textit{Separability.} The naïve GT mechanism faces a delicate shortcoming; for the GT to work, the reward function $r(\cdot;\btheta(t))$ must satisfy a $C$-separability assumption, which is stronger than Assumption~\ref{assumption: C separability}. Specifically, under $C$-separability,
any two arms $s\in\mcS^\star(\btheta(t))$ and $\tilde{s}\notin \mcS^\star(\btheta(t))$, and any set $\mcS\in[m]\setminus\{s,\tilde{s}\}$ must satisfy
\begin{align}
\label{eq:C_sep}
    r(\mcS\cup\{s\}\;;\;\btheta(t)) - r(\mcS\cup\{\tilde{s}\}\;;\;\btheta(t))\;\geq\;C\ .
\end{align}
Based on $C$-separability, the number of tests required for identifying $\mcS^\star(\btheta(t))$
will then be inversely proportional to $C^2$~\citep{zhou2014parallel}. However, it is impossible to ensure $C$-separability for the function $r(\cdot\;;\;\btheta(t))$ at round $t\in\N$, even when the reward function $r(\cdot\;;\;\bmu)$ at the true mean $\bmu$ is $C$-separable. We empirically show that the cumulative number of non-separable instances increases with time $t$. Figure~\ref{fig:1}, for any $t$, shows the number of times the reward function evaluated at $s\leq t$ is non-separable. Here, by ``non-separable'', we mean that the reward difference is smaller than $C$, i.e.,  $r(\mcS\cup\{s\}\;;\;\btheta(t)) - r(\mcS\cup\{\tilde{s}\}\;;\;\btheta(t))\leq C$, where $C$ is the minimum separability at the true mean. Furthermore, in~Figure~\ref{fig:1} we plot the function $\frac{1}{\Delta_{\min}^2(\bmu)}\log(t)$ and observe that the cumulative number of non-separable instances grows faster than $\frac{1}{\Delta_{\min}^2(\bmu)}\log(t)$, which is not desirable, as it can result in sub-optimal regret. 

\begin{algorithm}[!tb]
			\SetAlgoLined
			\LinesNumbered
			\SetKwInOut{Input}{Input}
			\Input{Cardinality constraint $K$, $\#$ rounds $T$}
			 
			\textit{Initialize}  $t=1$, $a_i(t) = 1$, $b_i(t) = 1$ for all $i\in[m]$
			
			\For{$t = 1 \ldots T$}{
			
			Draw a sample $\theta_i(t)\sim{\sf Beta}(a_t(t),b_i(t))$ for every arm $i\in[m]$, and form $\btheta(t)$
			
			{Play the super-arm} $\mcS(t)$ returned by ${\rm Oracle}(\btheta(t))$ (Algorithm~\ref{alg:GTO})
			
			Obtain the observations $Q(t)$

            Update $a_i(t+1)$ \& $b_i(t+1)$ according to~\eqref{eq: posterior1} ~\eqref{eq: posterior2}
		}
			\caption{GT+QTS Algorithm}
			\label{alg:CMAB-GT}
		\end{algorithm}
		
		\begin{algorithm}[!tb]
			\SetAlgoLined
			\LinesNumbered
			\SetKwInOut{Input}{Input}
			\Input{Parameter $\btheta$, quatization level $\Delta$, cardinality $K$, parameter $p$}
			 
			\textit{Initialize} $\#$ tests $\ell = O(\frac{1}{\hat q^2(t)\Delta^2}\log m)$,  Test matrix $\bA\in\{0,1\}^{\ell\times m}$ such that $A_{i,j} \sim{\sf Bern}(p)$

            \For{$j = 1 \ldots \ell$}{

                Evaluate the average reward function $r(\bA_j,\btheta)$ at the input $\btheta$

                Assign the quantized score $Q(r(\bA_j\;;\;\btheta))$ to the test $j$ according to~\eqref{eq: quantizer}

            }

            Evaluate the grading function using the decoding matrix $\bA$ according to~\eqref{eq:decoder}

            \SetKwInOut{Output}{Output}\Output{$\mcS(t):$ arms having the top-$K$ grades}
   
			\caption{${\rm Oracle}(\btheta)$}
			\label{alg:GTO}
		\end{algorithm}

\paragraph{Quantization.} To circumvent the non-separability of the reward function at the posterior means, we use  \emph{quantized rewards} as the test scores for GTO. Specifically, we use a uniform quantizer to discretize the reward values. This quantizer splits the interval $[0,M]$ into equal sub-intervals of width $\Delta/2B$, \rev{where the quantization level $\Delta$ will be specified in Section~\ref{sec:results} to ensure sublinear regret}\footnote{If $2BM/\Delta$ is not an integer, we absorb the missing fraction in the last interval, making it shorter than the preceding ones.}. Based on this, we split the reward range into $L = \lceil 2BM/\Delta\rceil$ intervals, where each interval $k\in[L-1]$ is defined as 
\begin{align}
    I_j  &:= \left (\frac{(i-1)\Delta}{2B},\frac{i\Delta}{2B}\right ]\ ,\quad  k<L\ ,\\
    I_K  &:= \left (\frac{(L-1)\Delta}{2B},M\right ]\ .
\end{align}
Furthermore, we denote the set of quantization levels by $\mcL := \{\Delta/2B, \cdots, M\}$. Accordingly, for any $\btheta\in[0,1]^m$ and $\mcS\in\mcI$, the uniform quantizer $Q : [0,B]\mapsto \mcL$ is specified by 
\begin{align}
\label{eq: quantizer}
    Q(r(\mcS\;;\;\btheta))\; := \;\argmin_{\ell\in\mcL}\; |r(\mcS\;;\;\btheta) - \ell|\ .
\end{align}

\paragraph{Decoding.} Note that quantization alone does not guarantee the separability of the reward defined in~\eqref{eq:C_sep} evaluated at {\em every} test. The reason is that the reward evaluations for the sets $\mcS\cup{\{s\}}$ and $\mcS\cup\{\tilde{s}\}$ may be mapped to the same quantization level, even though we have $r(\mcS\cup{\{s\}}\;;\;\btheta(t))>r(\mcS\cup\{\tilde{s}\}\;;\;\btheta(t))$ by Assumption~\ref{assumption: C separability}. Accordingly, at each round $t$, let us denote the set of {\em unique} (or non-repeated) test scores by $\mcI_{\sf nr}(t)$. Specifically, for any pair of distinct tests $\mcS,\mcS^\prime\in\mcI_{\sf nr}(t)$ such that $|\mcS|=|\mcS^\prime|$, it satisfies that $Q(r(\mcS\;;\;\btheta(t)))\neq Q(r(\mcS^\prime\;;\;\btheta(t)))$. In the decoding step, we leverage the fact that our quantization scheme enables us to sufficiently distinguish the tests contained in $\mcI_{\sf nr}(t)$.

\paragraph{Arm selection.} Let us denote the subset of arms obtained under a test matrix $\bA$ and the scoring function $\rho$ by ${\rm GTO}(\bA,\rho)$. At each round $t$, the GT+QTS algorithm uses the quantized average reward function $Q(r(\bA_i,\btheta(t)))$ as the scoring function for each test $i\in[\ell]$. Subsequently, the set of arms to be chosen at time $t$ is set to $\mcS(t) := {\rm GTO}(\bA,Q(r(\cdot,\btheta(t-1))))$. The entire algorithm is presented in Algorithm~\ref{alg:CMAB-GT}.    

\section{Main Results: Efficiency and Regret}
\label{sec:results}

In this section, we present the performance guarantees of the proposed GT+QTS algorithm. Specifically, we investigate two key performance metrics of the algorithm: (1) the efficiency of the GTO measured in terms of the number of reward evaluations required in each step, and (2) the average cumulative regret incurred by the GT+QTS algorithm. We show that the GT+QTS algorithm achieves the same order-wise regret guarantee as the combinatorial Thompson sampling using an exact oracle~\citep{wang2022thompson,perrault2021statistical}, while exponentially reducing the number of reward function evaluations. We begin with the results on the efficiency of the GTO. 

\paragraph{Efficiency.} A naïve approach to finding the optimal super-arm in each round is to evaluate the functional value at every subset in $\mcI$ at the current estimate of $\btheta(t)$ at time $t$. However, this approach requires an exponential number of reward evaluations. An exact oracle may not require an exponential number of evaluations, owing to the separability in Assumption~\ref{assumption: C separability}. We will first describe a baseline approach, called $\rm{Oracle}_+$, that provides an {\em exact} solution leveraging Assumption~\ref{assumption: C separability}, with the reward function evaluations scaling linearly with respect to the number of base arms. Subsequently, we will show that the GTO described in Section~\ref{sec:algo} finds the optimal arm with a high probability using only $O(\log m)$ function evaluations, exponentially reducing the complexity compared to the baseline approach. 

\paragraph{Oracle$_+$:} The baseline approach is a direct consequence of Assumption~\ref{assumption: C separability}. Since the separability assumption is valid for {\em any} subset $\mcS$, for any parameter $\btheta\in[0,1]^m$ we may set $\mcS = \emptyset$. By this choice, for any $s\in\mcS^\star(\btheta)$ and $\tilde{s}\notin \mcS^\star(\btheta)$, Assumption~\ref{assumption: C separability} implies that
\begin{align}
\label{eq:baseline}
    r(s\;;\;\btheta)\; > r(\tilde{s}\;;\;\btheta)\ .
\end{align}
${\rm Oracle}_+$ makes $m$ reward evaluations, each test comprising of a single base arm. It then selects the top $K$ arms with the largest reward values. As a consequence of~\eqref{eq:baseline}, we immediately conclude that this set of base arms selected by ${\rm Oracle}_+$ is indeed the optimal super-arm $\mcS^\star(\btheta)$. Therefore, ${\rm Oracle}_+$ requires $m$ (linear) reward function evaluations. Next, we analyze the number of tests required by the GTO. 

\paragraph{GTO:} For any separable function, the number of tests required by the GTO is of the order $O(\log m)$, where the constants depend on the \rev{quantization level $\Delta$}, the set $\mcI_{\sf nr}(t)$, as well as the test matrix parameter $p$. For characterizing the number of tests required by GTO, let us denote the probability of the set of non-repeated test scores by
\begin{align}
    q(t):=\P(\mcI_{\sf nr}(t))\ .
\end{align}
The following lemma formalizes the number of tests the GTO requires to compute the optimal super-arm at each time $t\in\N$. 
\begin{lemma}
\label{lemma: number of tests}
    For any $\delta\in(0,1)$, 
    \begin{align}
    \label{eq: number of tests}
        \ell\;=\; \frac{8M^2B^2}{\Delta^2p^4(1-p)^2q^2(t)}\log \left ( \frac{K(m-K)}{\delta}\right )
    \end{align}
    tests are sufficient for the GTO to identify an optimal super-arm in each round with probability at least $1-\delta$.
\end{lemma}

From~\eqref{eq: number of tests}, we observe that the GTO requires $O(\log m)$ tests to identify the optimal super-arm in each round with probability at least $1-\delta$. Hence, in the regime of a large number of base arms, the GTO {\em significantly} reduces the number of reward evaluations required to find the optimal super-arm in each round. We also observe that the number of tests depends on $q(t)$ which is unknown. This can be resolved by adopting an estimator for estimating $q(t)$ based on the group tests. Let us define
\begin{align}
    \hat q(t)\;:=\;\frac{1}{\ell}\sum\limits_{j\in[\ell]} \mathds{1}\{\bA_j\in\mcI_{\sf nr}(t)\}\ .
\end{align}
We show that using $O(\frac{1}{\varepsilon^2}\log\frac{1}{\delta})$ samples, we have an $\varepsilon-$accurate estimate of $q(t)$ with a high probability, i.e., $\P(|\hat q(t) - q(t)|>\varepsilon)\leq \delta$. Both $\Delta$ and $q(t)$ capture the granularity of the reward function in identifying the optimal super-arm. \rev{We will show that $\Delta$ is chosen based on the CDF $\mathbb{F}_{\bmu}$ of the arm gaps $\Delta_{\min}(\bmu)$, which captures the gap between reward due to optimal arm and any other arm. Smaller the quantization level $\Delta$, more tests are required to find the optimal super-arm.} Similarly, if we face a reward function in which many tests get mapped to the same score, it is unlikely that we save much leveraging group testing.

\paragraph{Regret analysis.} Next, we characterize the regret of the GT+QTS algorithm. As a first step, we establish that our quantization scheme for super-arm selection does not compromise the achievable regret. In other words, the regret achieved after reward quantization is the same as the regret with unquantized rewards. For this, we introduce a few notations. Let $\mcT(\btheta)$ denote the {\em set} of \emph{all} optimal super-arms with respect to the parameter $\btheta$, i.e., $\mcS^\star(\btheta)\in\mcT(\btheta)$. Next, corresponding to the function $Q$ specified in~\eqref{eq: quantizer} we define
\begin{align}    \mcT_Q(\btheta)\;:=\;\argmax\limits_{\mcS\in\mcI} Q(r(\mcS\;;\;\btheta))\ .
\end{align}
We show that, \rev{with a high probability}, the set of optimal super-arms with respect to the quantized reward $\mcT_Q(\bmu)$ is contained in the set of optimal super-arms with respect to the true reward $\mcT(\bmu)$. This is necessary to achieve sublinear regret since we hope to converge to one of the optimal super-arms after quantization. 

\begin{lemma}
\label{lemma:optimal_quantized}
    \rev{For any $\gamma\in[0,1/2]$, let us set $\Delta := \mathbb{F}_{\bmu}^{-1}(\gamma)$. For the quantization scheme described in Section~\ref{sec:algo}, with probability at least $1-2\gamma$ we have $\mcT_Q(\bmu)\subseteq \mcT(\bmu)$.}
\end{lemma}
Next, we provide an upper bound on the average cumulative regret achievable by GT+QTS. We note that this is similar to the regret bound reported when there is access to an exact oracle for the CTS algorithm. 
\begin{theorem}[Achievable regret]
\label{theorem: regret upper bound}
    Under Assumptions~\ref{assumption: mean only}--\ref{assumption: C separability}, by setting $\delta = \frac{1}{t^2}$, \rev{and the quantization level $\Delta:= \mathbb{F}^{-1}_{\bmu}(\gamma)$ for any $\gamma\in[0,1/2]$, with probability at least $1-2\gamma$, the regret of the GT+QTS algorithm satisfies}
    {\small
    \begin{align}
        \mathfrak{R}(T)\; &\leq \;\sum\limits_{i\in[m]} \left ( 2\log K + 6\right )B^2\nonumber\\
        &\qquad\times\frac{\log(2^m|\mcI|T)}{\min\limits_{\mcS : i\in\mcS}\big (\Delta(\mcS,\bmu) - \frac{\Delta_{\min}(\bmu)}{2}-(K^2+2)B\varepsilon \big)}\nonumber\\
        &\qquad+\Bigg(13\alpha\frac{8}{\varepsilon^2}\bigg ( \frac{4}{\varepsilon^2} + 1\bigg )^K \log\frac{K}{\varepsilon^2} + \frac{\pi^2}{6}\nonumber\\
        &\qquad\qquad\qquad+m\bigg(\frac{K^2}{\varepsilon^2} + 1\bigg )\Bigg )\Delta_{\max}(\bmu)\ ,
        \label{eq: theorem1}
    \end{align}}
    where $\alpha\in\R_+$ is a constant, and $\varepsilon\in\R_+$ is chosen as
    \begin{align}
    \label{eq: varepsilon}
        \varepsilon\;<\;\frac{\Delta_{\min}(\bmu)}{4B(K^2 + 2)}\ .
    \end{align}
\end{theorem}
The regret bound in Theorem~\ref{theorem: regret upper bound} matches that of the CTS algorithm with an exact oracle~\citep{wang2022thompson} order-wise, i.e., both have the same regret bound of $O(m\Delta^{-1}_{\min}(\bmu)\log K \log T)$, despite GT+QTS requiring exponentially fewer reward function evaluations.  We note that the regret remains linear in $m$. The numerator of the first term in the summand, i.e., $\log(2^m|\mcI|T)$, can be decomposed into two parts. The first part is $m\log(2|\mcI|)$, and the second part is $\log T$. The first part depends on $T$ through $\log T$, but it is {\bf independent of $m$}, and second part is {\bf independent of $T$} but depends on $m$ linearly.
Since the second part is independent of $T$ it does not contribute to the regret, and therefore the regret will be specified only by the first term. In other words, after summing both terms $m$ times, we get the total regret of $m\log T+m^2$, which is $O(m\log T)$. Comparing the bound in Theorem~\ref{theorem: regret upper bound} to that of~\citet[Theorem 1]{wang2022thompson}, we observe that GTO only adds a constant term of $\frac{\pi^2}{6}\Delta_{\max}(\bmu)$ to the regret bound.

\begin{proof}
We provide an overview of the key steps and defer the details of the proof to Appendix~\ref{proof: regret upper bound}. From the definition of regret in~\eqref{eq:regret}, with probability at least $1-2\gamma$ we have
    \begin{align}\nonumber
    \mathfrak{R}(T)\; &= \;\sum\limits_{t=1}^T \E[\mathds{1}\{\mcS(t)\notin\mcT(\bmu)\}\times\Delta(\mcS(t),\bmu)]\\
    \label{eq:regret_lemma_quant}
    &\leq\;\sum\limits_{t=1}^T \E[\mathds{1}\{\mcS(t)\notin\mcT_Q(\bmu)\}\times\Delta(\mcS(t),\bmu)]\ ,
\end{align}
where~\eqref{eq:regret_lemma_quant} is a result of Lemma~\ref{lemma:optimal_quantized}. Next, 
we decompose the upper bound on the regret in~\eqref{eq:regret_lemma_quant} based on three events. The first event captures the time instances at which we select a sub-optimal super-arm owing to inaccurate sample mean estimates. The second event considers time instances at which the sample mean is close to the true mean, and yet we select a sub-optimal super-arm since our posterior mean has a considerable deviation from the true mean. Finally, the third event considers the instances where the posterior mean of the {\em selected} super-arm is close to the true mean, and yet, we select a sub-optimal super-arm. The key challenge arises in upper-bounding the regret due to this third event. Specifically, the proof for this term in~\citet{wang2022thompson} relies on assuming access to an {\em exact} oracle -- an assumption that we have dispensed with. We  show that GTO is sufficient to guarantee constant regret for the third set of events.
\end{proof}


\section{Experiments}
\label{sec:exp}

In this section, we provide empirical results to assess the performance of GT+QTS and compare it against the state-of-the-art CTS algorithm provided in~\citep{wang2022thompson} equipped with ${\rm Oracle}_+$ described in Section~\ref{sec:results} as the exact oracle. In the first experiment, we consider the case of linear rewards. Subsequently, we consider the mean reward to be the output of a $2$-layer artificial neural network.

\begin{figure*}[tb]
    \centering
\begin{tabular}{ccc}
    \includegraphics[width=0.31\textwidth,trim={1cm 0cm  1cm  0cm },clip]{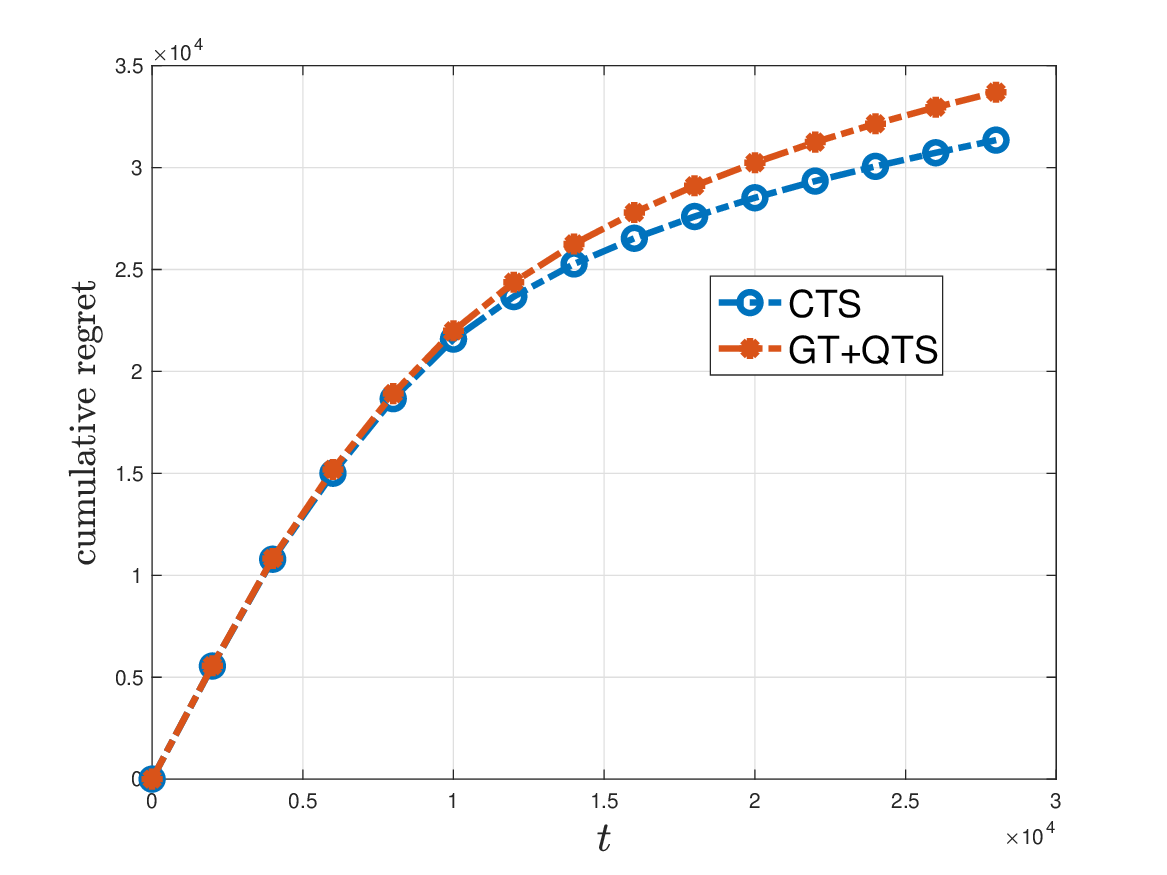} &
    \includegraphics[width=0.31\textwidth,trim={1cm 0cm  1cm  0cm },clip]{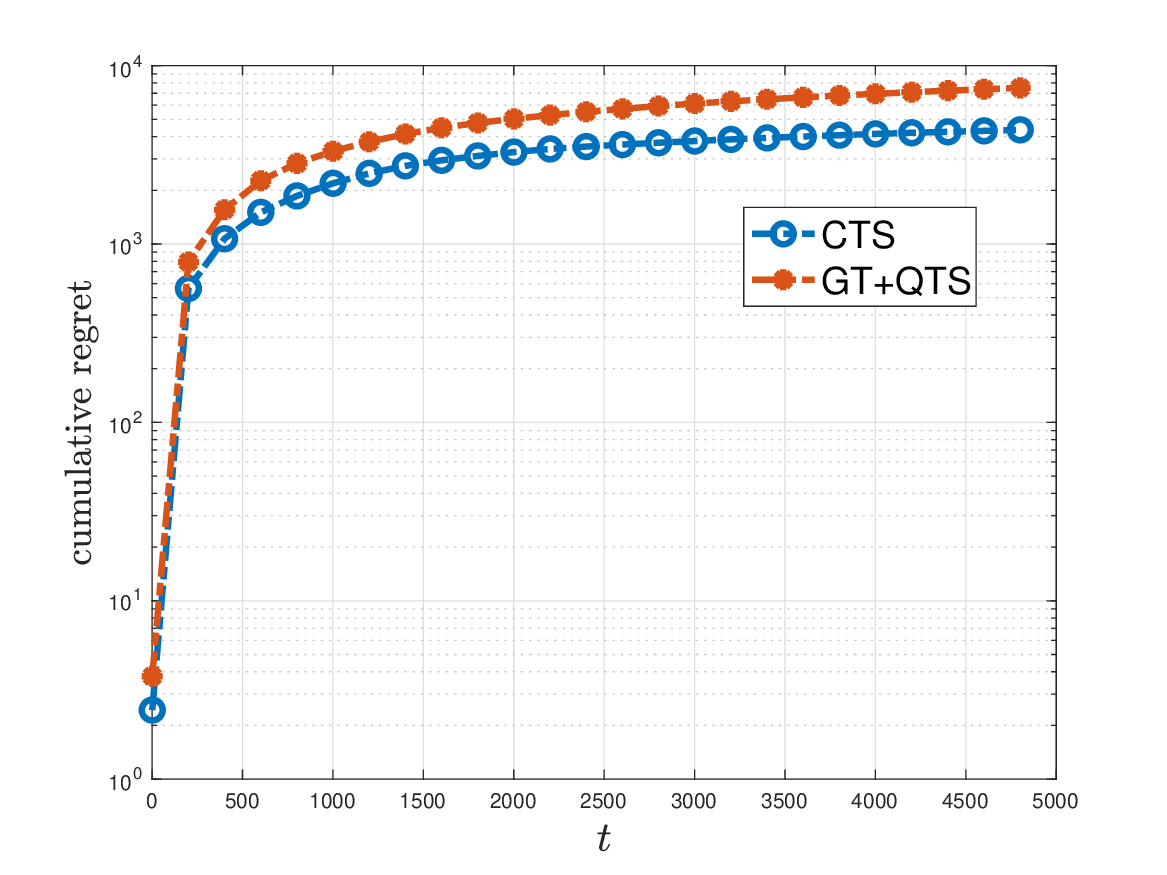} & 
    \includegraphics[width=0.31\textwidth,trim={1cm 0cm  1cm  0cm },clip]{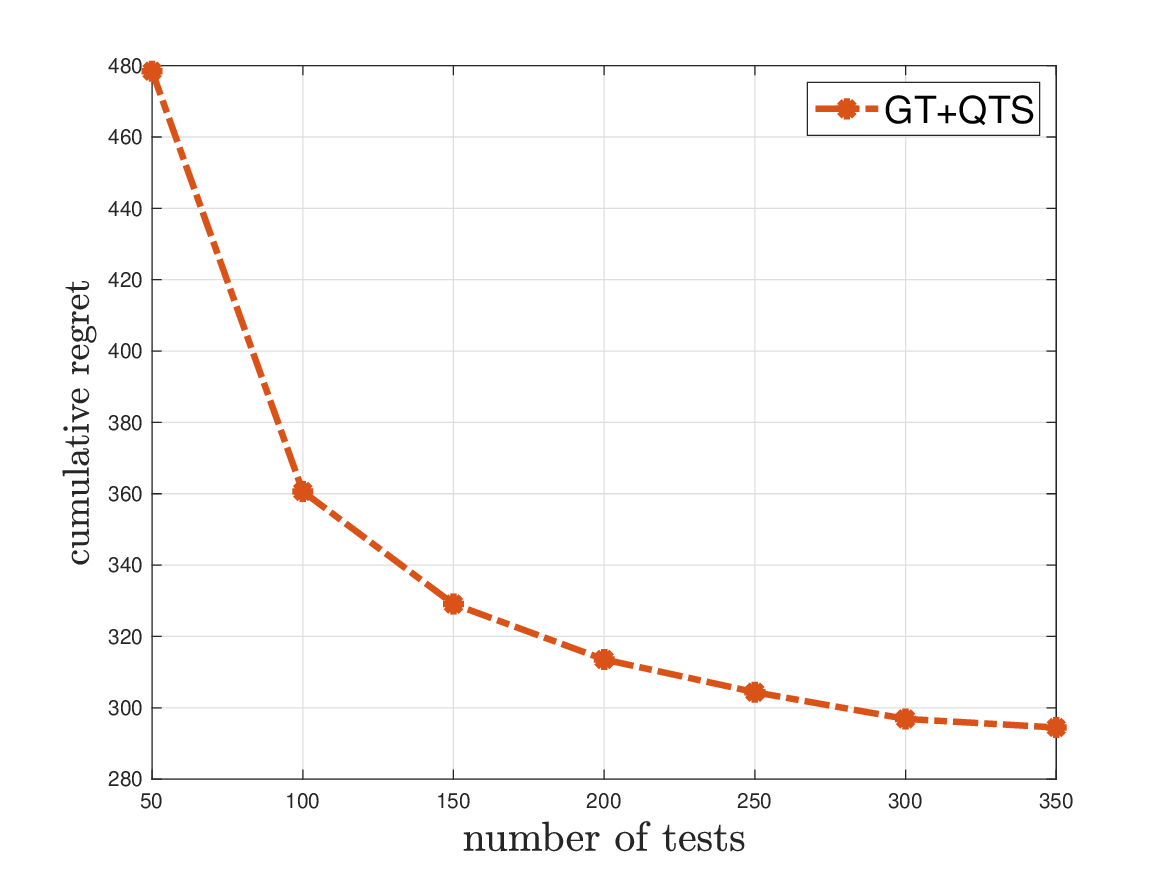} \\
   (A)  & (B) & (C)
\end{tabular}
        \caption{Average cumulative regret versus time $t$ for (A) a linear reward function and (B)a non-linear 2-layer ANN function, respectively. (C) Average cumulative regret versus number of tests $\ell$.}
        \label{fig:2}
\end{figure*}

\paragraph{Linear rewards.}
In this experiment, for any set $\mcS\in\mcI$ and any $\btheta\in[0,1]^m$, we define the reward function as $r(\mcS\;;\;\btheta) = \sum_{i\in\mcS}\theta_i$. We set $m=5000$ arms, and the mean vector $\bmu$ is sampled uniformly randomly from $[0,1]^{5000}$. Furthermore, the cardinality constraint is set to $K=5$ arms. For this experiment, we choose $\bmu$ such that $\Delta_{\min}(\bmu)$ is at least $0.25$, \rev{and we set $\Delta = 0.25$}. Consequently, the group testing oracle requires approximately $302$ reward evaluations (order-wise), versus the exact oracle, which requires $5000$ reward evaluations in each iteration. Hence, the baseline method (CTS) requires $16\times$ more reward evaluations compared to GT+QTS. However, the regret due to CTS and GT+QTS is comparable, and GT+QTS has a slightly larger regret compared to CTS, as observed in Fig.~\ref{fig:2}(A). 

Furthermore, we empirically observe that the number of tests prescribed by theory is excessive, and in practice, much fewer tests are sufficient to guarantee similar cumulative regret. To showcase this, we vary the number of group tests in the GT+QTS algorithm and plot the average cumulative regret against the number of group tests in Figure~\ref{fig:2}(C) computed at $T=10,000$. Figure~\ref{fig:2}(C) confirms that as few as $\ell=200$ tests are sufficient for the regret to be within $5\%$ of the regret at the prescribed number of $\ell\approx 310$ tests.

\begin{table}[t]
	\caption{Average computation time per trial (in secs)}
	\label{tab:1}
	\centering
	\resizebox{0.49\textwidth}{!}{%
		\begin{tabular}{|c|c|c|c|c|c|}
			\hline
			Horizon & $T=1000$     & $T=2000$  & $T=3000$      & $T=4000$  &$T=5000$     \\ \hline
			CTS      & $222.7$ & $458.5$ & $723.8$ & $972.6$ & $1172.5$ \\ \hline 
            GT+QTS      & $\mathbf{55.95}$ & $\mathbf{120.21}$ & $\mathbf{173.52}$ & $\mathbf{229.80}$ & $\mathbf{287.35}$ \\ \hline 
	\end{tabular}}
\end{table}

\paragraph{Two-layer neural network rewards.} Next, we evaluate the performance of GT+QTS on nonlinear mean reward functions. Specifically, we choose a $2$-layer NN with $20$ neurons and sigmoid activation function. For any set $\mcS$ and parameter $\btheta$, the mean reward is $r(\mcS\;;\;\btheta) = \langle \bw_2,\sigma(\bW_1\btheta_{\mcS})\rangle$, where $\sigma(\cdot)$ denotes the sigmoid activation. The weights are uniformly sampled from a normal distribution, and then we take an absolute value to make all weights positive. We choose $m=1000$ arms, which are uniformly sampled at random from $[0,1]^{1000}$. Furthermore, the weight matrices are sampled such that we have $\Delta_{\min}(\bmu)\geq 0.2$, \rev{and we set $\Delta=0.2$}. Figure~\ref{fig:2}(B) illustrates the cumulative regret of the CTS algorithm and the GT+QTS algorithm, which follow the same order in $T$. Furthermore, to assess the gain obtained as a result of the GTO, we report the average time required by the CTS and GT+QTS algorithms for different values of the horizon $T$ in Table~\ref{tab:1}. We observe that GT+QTS is {\em significantly} more efficient compared to CTS.

\subsection*{Conclusions}
We investigate the problem of regret-minimization in combinatorial semi-bandits. Existing approaches assume the existence of an exact oracle, which may not always be computationally viable. To circumvent this issue, we establish a novel connection between group testing and combinatorial bandits. We propose a new arm-selection strategy that combines a group testing oracle with a Thompson sampling-based super-arm selection strategy. Under a \rev{probabilistic assumption on the minimum separation over the class of bandit instances}, the proposed GT-QTS algorithm has two key advantages: 1) it is significantly more efficient compared to the exact oracle since it requires exponentially fewer reward evaluations at each step, and 2) it preserves the regret guarantee of the state-of-the-art method order-wise. We provide numerical evaluations to bolster our analytical claims.


  \newpage
 \appendix

\def\toptitlebar{
\hrule height4pt
\vskip .25in}

\allowdisplaybreaks

\section{Proof of Lemma~\ref{lemma: number of tests}}
\label{proof: lemma 1}
At any instant $t\in \N$, choose an optimal arm $s\in\mcS^\star(\btheta(t))$ and a sub-optimal arm $\tilde{s}\notin\mcS^\star(\btheta(t))$. For accurate prediction, the 
arm grade $\phi_s(t)$ for arm $s$ should be more than $\phi_{\tilde{s}}(t)$ assigned to arm $\tilde{s}$. Let us denote the $i^{\sf th}$ column of any matrix $\bA$ by $\bA_{:,i}$. Finding the difference between the arm grades, we have
\begin{align}
    \phi_s(t) - \phi_{\tilde{s}}(t)\; &= \;\left\langle \bA_{:,s}\;,\;\rho(t)\right\rangle - \left\langle \bA_{:,\tilde{s}}\;,\;\brho(t)\right\rangle\\
    & = \sum\limits_{j=1}^{\ell} \underbrace{\left ( \bA_{j,s} - \bA_{j,\tilde{s}}\right )\rho_j(t)}_{:=Z_j(t)}\ .
\end{align}
Furthermore, we have
\begin{align}
    \E[Z_j(t)]\;=&\; \E\left [ (\bA_{j,s} - \bA_{j,\tilde{s}})\rho_j(t)\right ]\\
    & =\; \E\left [ (\bA_{j,s} - \bA_{j,\tilde{s}})\times Q(r(\bA_j\;;\;\btheta(t)))\right ]\\
    &=\; \sum\limits_{\mcS\subseteq[m]} \left ( \mathds{1}\{s\in\mcS\} - \mathds{1}\{\tilde{s}\in\mcS\}\right )\times Q(r(\mcS\;;\;\btheta(t)))\times \P(\mcS\in\bA)\\
    &=\; \sum\limits_{\mcS\subseteq[m] : s\in\mcS,\tilde{s}\in\mcS} \left ( \underbrace{\mathds{1}\{s\in\mcS\} - \mathds{1}\{\tilde{s}\in\mcS\}}_{=0}\right )\times Q(r(\mcS\;;\;\btheta(t)))\times \P(\mcS\in\bA)\nonumber\\
    &\;\;\;+\sum\limits_{\mcS\subseteq[m] : s\notin\mcS,\tilde{s}\notin\mcS} \left ( \underbrace{\mathds{1}\{s\in\mcS\} - \mathds{1}\{\tilde{s}\in\mcS\}}_{=0}\right )\times Q(r(\mcS\;;\;\btheta(t)))\times \P(\mcS\in\bA)\nonumber\\
    &\;\;\;\;\;+\sum\limits_{\mcS\subseteq[m] : s\in\mcS,\tilde{s}\notin\mcS} \left ( \underbrace{\mathds{1}\{s\in\mcS\} - \mathds{1}\{\tilde{s}\in\mcS\}}_{=1}\right )\times Q(r(\mcS\;;\;\btheta(t)))\times \P(\mcS\in\bA)\nonumber\\
    &\;\;\;\;\;\;\;\;+\sum\limits_{\mcS\subseteq[m] : s\notin\mcS,\tilde{s}\in\mcS} \left ( \underbrace{\mathds{1}\{s\in\mcS\} - \mathds{1}\{\tilde{s}\in\mcS\}}_{=-1}\right )\times Q(r(\mcS\;;\;\btheta(t)))\times \P(\mcS\in\bA)\\
    &=\; \sum\limits_{\mcS\subseteq [m]\setminus\{s,\tilde{s}\}} Q \Big ( r(\mcS\cup\{s\}\;;\;\btheta(t))\Big)\times\P(\mcS\cup\{s\}\in\bA)\nonumber\\
    &\qquad\qquad - \sum\limits_{\mcS\subseteq [m]\setminus\{s,\tilde{s}\}} Q \Big ( r(\mcS\cup\{\tilde{s}\}\;;\;\btheta(t))\Big)\times\P(\mcS\cup\{\tilde{s}\}\in\bA)\\
    &=\; p(1-p)\sum\limits_{\mcS\subseteq[m]\setminus\{s,\tilde{s}\}} \left ( Q \Big ( r(\mcS\cup\{s\}\;;\;\btheta(t))\Big) - Q \Big ( r(\mcS\cup\{\tilde{s}\}\;;\;\btheta(t))\Big)\right) \nonumber\\
    &\qquad\qquad\qquad\qquad\qquad\times p^{|\mcS|}(1-p)^{m-|\mcS|-2}\ .
    \label{eq:tests1}
\end{align}

Next, let us recall the definitions of the set of repeated tests $\mcI_{\sf nr}$(t). Accordingly, we have
\begin{align}
    \E[Z_j(t)]\;&=\; p(1-p)\sum\limits_{\mcS\subseteq[m]\setminus\{s,\tilde{s}\}} \left ( Q \Big ( r(\mcS\cup\{s\}\;;\;\btheta(t))\Big) - Q \Big ( r(\mcS\cup\{\tilde{s}\}\;;\;\btheta(t))\Big)\right) \nonumber\\
    &\qquad\qquad\qquad\qquad\qquad\times p^{|\mcS|}(1-p)^{m-|\mcS|-2}\\
    &=\;p(1-p)\sum\limits_{\mcS\subseteq[m]\setminus\{s,\tilde{s}\}:\mcS\cup\{s\}\in\mcI_{\sf nr}(t)} \underbrace{\left ( Q \Big ( r(\mcS\cup\{s\}\;;\;\btheta(t))\Big) - Q \Big ( r(\mcS\cup\{\tilde{s}\}\;;\;\btheta(t))\Big)\right)}_{\geq \frac{\Delta_{\min}(\bmu)}{2B}} \nonumber\\
    &\qquad\qquad\qquad\qquad\qquad\times p^{|\mcS|}(1-p)^{m-|\mcS|-2}\nonumber\\
    &\;+p(1-p)\sum\limits_{\mcS\subseteq[m]\setminus\{s,\tilde{s}\}:\mcS\cup\{s\}\notin\mcI_{\sf nr}(t)} \underbrace{\left ( Q \Big ( r(\mcS\cup\{s\}\;;\;\btheta(t))\Big) - Q \Big ( r(\mcS\cup\{\tilde{s}\}\;;\;\btheta(t))\Big)\right)}_{\geq 0} \nonumber\\
    &\qquad\qquad\qquad\qquad\qquad\times p^{|\mcS|}(1-p)^{m-|\mcS|-2}\\
    &\geq\; \frac{\Delta}{2B}p(1-p) \sum\limits_{\mcS\subseteq[m]\setminus\{s,\tilde{s}\} : \mcS\cup\{s\}\in\mcI_{\sf nr}(t)} p^{|\mcS|}(1-p)^{m-|\mcS|-2}\\
    &=\;\frac{\Delta}{2B}p(1-p)\P\Big (\{\mcS\in\mcI_{\sf nr}(t) : s\in\mcS\} \Big )\\
    &=\;\frac{\Delta}{2B}p^2(1-p) q(t)\ ,
    \label{eq:tests2}
\end{align}
where~\eqref{eq:tests2} follows from the fact that $\P(\mcS\in\mcI_{\sf nr}(t), s\in\mcS) = \P(\mcS\in\mcI_{\sf nr}(t)\med s\in\mcS)\P(s\in\mcS)=pq(t)$. Furthermore, since we have $Z_j(t)\in[-M,M]$ as per Assumption~\ref{assumption: bounded reward}, by Hoeffding's inequality we have
\begin{align}
    \P\Big ( \phi_s(t) - \phi_{\tilde{s}}(t)\leq 0\Big )\;\leq\;\exp\left ( -\frac{\ell\Delta^2p^4(1-p)^2q^2(t)}{8M^2B^2}\right )\ .
    \label{eq: tests5}
\end{align}
Finally, noting that there are $K(m-K)$ possible ways to choose $s$ and $\tilde{s}$, taking a union bound along with~\eqref{eq: tests5} concludes the proof.

\textbf{Estimating $q$}: First, note that $\hat q(t)$ is an unbiased estimator of $q(t)$. This is because
\begin{align}
    \E[\hat q(t)]\;&=\; \frac{1}{\ell}\E\left[ \sum\limits_{j\in\ell} \mathds{1}\{\bA_j\in\mcI_{\sf nr}(t)\}\right]\\
    &=\; \frac{1}{\ell}\sum\limits_{j\in\ell} \P\Big ( \bA_j\in\mcI_{\sf nr}(t)\Big )\\
    &=\; \P\Big ( \mcI_{\sf nr}(t)\Big )\\
    &=\; q(t)\ .
\end{align}
Furthermore, since $\hat q(t)$ is an unbiased estimator of $q$, using the Hoeffding's inequality, we obtain that for any $\varepsilon\in\R_+$ and $\delta\in(0,1)$,
\begin{align}
    \ell\;=\; \frac{1}{2\varepsilon^2}\log\frac{1}{\delta}
\end{align}
tests are sufficient to ensure that 
\begin{align}
    \P\Big( |\hat q(t) - q|>\varepsilon\Big)\;\leq\;\delta\ .
\end{align}

\section{Proof of Lemma~\ref{lemma:optimal_quantized}}

    \rev{First, we will show that with a high probability, we have  $\mcT_{Q}(\bmu)\cap\mcT(\bmu)\neq\emptyset$. Note that 
    \begin{align}
        &\P\Big(  \mcT_{Q}(\bmu)\cap\mcT(\bmu) = \emptyset\Big)\nonumber\\
        &=\P\Big( \exists\:\mcS\in\mcT(\bmu),\exists\: \mcS^\prime\in\mcT_{Q}(\bmu) : \mcS\notin\mcT_{Q}(\bmu)\;\text{and}\; \mcS^\prime\notin \mcT(\bmu)\Big)\\
        &\leq \P\Big(\exists\:\mcS\in\mcT(\bmu),\exists\: \mcS^\prime\in\mcT_{Q}(\bmu) :  r(\mcS\;;\;\bmu) - r(\mcS^\prime\;;\;\bmu) \geq \Delta_{\min}(\bmu) \Big)\\
        &=\P\Big ( \exists\:\mcS\in\mcT(\bmu),\exists\: \mcS^\prime\in\mcT_{Q}(\bmu) :\nonumber\\
        &\qquad r(\mcS\;;\;\bmu) - Q(r(\mcS^\prime\;;\;\bmu)) + Q(r(\mcS^\prime\;;\;\bmu)) - r(\mcS^\prime\;;\;\bmu) \geq \Delta_{\min}(\bmu)\Big)\\
        \label{eq:quantization_1}
        &\leq\P\Big ( \exists\:\mcS\in\mcT(\bmu),\exists\: \mcS^\prime\in\mcT_{Q}(\bmu) :\nonumber\\
        &\qquad r(\mcS\;;\;\bmu) - Q(r(\mcS\;;\;\bmu)) + Q(r(\mcS^\prime\;;\;\bmu)) - r(\mcS^\prime\;;\;\bmu) \geq \Delta_{\min}(\bmu)\Big)\\
        \label{eq:quantization_2}
        &\leq\P\bigg( \exists\:\mcS\in\mcT(\bmu),\exists\: \mcS^\prime\in\mcT_{Q}(\bmu) :\frac{\Delta}{4B} + \frac{\Delta}{4B}\geq \Delta_{\min}(\bmu)\bigg)\\
        &\leq\P\bigg( \frac{\Delta}{2B}\geq \Delta_{\min}(\bmu)\bigg)\\
        \label{eq:quantization_2a}
        &=\P\bigg( \frac{\mathbb{F}_{\bmu}^{-1}(\gamma)}{2B}\geq \Delta_{\min}(\bmu)\bigg)\\
        \label{eq:quantization_3}
        &\leq \P\Big(\mathbb{F}_{\bmu}^{-1}(\gamma) \geq \Delta_{\min}(\bmu) \Big)\\
        \label{eq:quantization_3a}
        &=\gamma\ ,
    \end{align}
    where~\eqref{eq:quantization_1} follows from the definition of the set $\mcS^\prime$, \eqref{eq:quantization_2} follows from the quantization scheme in~\eqref{eq: quantizer}, \eqref{eq:quantization_2a} holds since we have set $\Delta = \mathbb{F}_{\bmu}^{-1}(\gamma)$, and \eqref{eq:quantization_3} holds since $B$ is the Lipschitz constant in Assumption~\ref{assumption:Lipschitz}, and it can always be set to be larger than $1/2$, if any $B<1/2$ satisfies Assumption~\ref{assumption:Lipschitz}.}
    
    \rev{This proves that $\mcT_Q(\bmu)\cap\mcT(\bmu)\neq \emptyset$ with probability at least $1-\gamma$. Next, following a similar line of arguments, we will show that $\mcT_{Q}(\bmu)\subseteq \mcT(\bmu)$ with a high probability. Let us define the event
    \begin{align}
        \mcE(\bmu)\;:=\;\Big\{ \mcT_{Q}(\bmu)\cap \mcT(\bmu) = \emptyset\Big\}\ .
    \end{align}
    We have,
    \begin{align}
        &\P\big( \mcT(\bmu)\subset\mcT_Q(\bmu)\Big)\nonumber\\
        &=\;\P\Big( \exists\: \mcS^\prime\in \mcT_Q(\bmu) : \mcS^\prime\notin\mcT(\bmu)\Big)\\
        &=\;\P\Big( \exists\: \mcS^\prime\in \mcT_Q(\bmu) : \mcS^\prime\notin\mcT(\bmu)\med\mcE(\bmu)\Big)\P\big(\mcE(\bmu)\big)+ \P\Big( \exists\: \mcS^\prime\in \mcT_Q(\bmu) : \mcS^\prime\notin\mcT(\bmu)\med\overline{\mcE(\bmu)}\Big)\P\big(\overline{\mcE(\bmu)}\big)\\
        &\stackrel{\eqref{eq:quantization_3a}}{<}\; \P\Big( \exists\: \mcS^\prime\in \mcT_Q(\bmu) : \mcS^\prime\notin\mcT(\bmu)\med\overline{\mcE(\bmu)}\Big) + \gamma\\
        &=\;\P\Big ( \exists\:\mcS^\prime\in\mcT_Q(\bmu), \tilde\mcS\in\mcT_Q(\bmu)\cap\mcT(\bmu) : r(\tilde{\mcS}\;;\;\bmu) - r(\mcS^\prime\;;\;\bmu) \geq \Delta_{\min}(\bmu)\med \overline{\mcE(\bmu)}\Big ) + \gamma\\
        &=\;\P\Big ( \exists\:\mcS^\prime\in\mcT_Q(\bmu), \tilde\mcS\in\mcT_Q(\bmu)\cap\mcT(\bmu) : r(\tilde{\mcS}\;;\;\bmu) - Q(r(\mcS^\prime\;;\;\bmu))\nonumber\\
        &\qquad\qquad + Q(r(\mcS^\prime\;;\;\bmu)) - r(\mcS^\prime\;;\;\bmu) \geq \Delta_{\min}(\bmu)\med \overline{\mcE(\bmu)}\Big ) + \gamma\\
        &=\;\P\Big ( \exists\:\mcS^\prime\in\mcT_Q(\bmu), \tilde\mcS\in\mcT_Q(\bmu)\cap\mcT(\bmu) : r(\tilde{\mcS}\;;\;\bmu) - Q(r(\tilde\mcS\;;\;\bmu))\nonumber\\
        &\qquad\qquad + Q(r(\mcS^\prime\;;\;\bmu)) - r(\mcS^\prime\;;\;\bmu) \geq \Delta_{\min}(\bmu)\med \overline{\mcE(\bmu)}\Big ) + \gamma\\
        &\leq\;\P\Big( \exists\:\mcS^\prime\in\mcT_Q(\bmu), \tilde\mcS\in\mcT_Q(\bmu)\cap\mcT(\bmu) : \frac{\Delta}{4B} + \frac{\Delta}{4B} > \Delta_{\min}(\bmu)\med \overline{\mcE(\bmu)}\Big) + \gamma\\
        &\leq\; \P\Big( \exists\:\mcS^\prime\in\mcT_Q(\bmu), \tilde\mcS\in\mcT_Q(\bmu)\cap\mcT(\bmu) : \frac{1}{2}\Delta > \Delta_{\min}(\bmu)\med \overline{\mcE(\bmu)}\Big) + \gamma\\
        \label{eq:quantization_4}
        &\leq\; \P\Big( \frac{1}{2}\Delta > \Delta_{\min}(\bmu)\Big)\\
        &\leq\; 2\gamma \ ,
    \end{align}
    where~\eqref{eq:quantization_4} follows from the fact that the events $\overline{\mcE(\bmu)}$ and $\{\frac{1}{2}\Delta_{\min}(\bmu)\}$ are independent of each other, since the distribution of $\Delta_{\min}(\bmu)$ is a property of the environment, and does not depend on the event $\mcE(\bmu)$. This concludes our proof.}

\section{Proof of Theorem~\ref{theorem: regret upper bound}}
\label{proof: regret upper bound}
Similarly to~\citep{wang2022thompson}, we begin by defining a few events that are instrumental in characterizing the upper bound on the average regret. First, let us denote the number of times that any arm $i\in[m]$ is sampled until time $t\in\N$ by $T_i(t)$. Furthermore, let us denote the sample mean for any arm $i\in[m]$ at time $t\in\N$ by $\bar{\mu}_i(t)$. Accordingly, let us define
\begin{enumerate}
    \item $\mcA(t)\;:=\;\{\mcS(t)\notin\mcT_Q(\bmu)\}$.
    \item $\mcB(t)\;:=\; \bigg\{\exists i\in\mcS(t) : |\bar\mu_i(t) - \mu_i| > \frac{\varepsilon}{|\mcS(t)|}\bigg\}$.
    \item $\mcC(t)\;:=\;\left\{||\btheta_{\mcS(t)}(t) - \bmu_{\mcS(t)}||_1 > \frac{\Delta(\mcS(t),\bmu)}{B} - \frac{\Delta_{\min}(\bmu)}{2B} - (K^2 + 2)\varepsilon\right\}$.
\end{enumerate}
\rev{With a probability at least $1-2\gamma$}, we can decompose the regret as follows.
\begin{align}
    \mathfrak{R}(T)\; &= \;\sum\limits_{t=1}^T \E\Big[\mathds{1}\{\mcS(t)\notin\mcT(\bmu)\}\times\Delta(\mcS(t),\bmu)\Big]\\
    \label{eq:regret_1}
    &\leq\;\sum\limits_{t=1}^T \E\Big [\mathds{1}\{\mcA(t))\}\times\Delta(\mcS(t),\bmu)\Big ]\\
    &\leq\;\underbrace{\sum\limits_{t=1}^T \E\Big[\mathds{1}\{\mcA(t)\cap\mcB(t)\}\times\Delta(\mcS(t),\bmu)\Big]}_{A_1}+ \underbrace{\sum\limits_{t=1}^T\E\Big[\mathds{1}\{\mcA(t)\cap\overline{\mcB(t)}\cap\mcC(t)\}\times\Delta(\mcS(t),\bmu)\Big]}_{A_2}\nonumber\\
    &\qquad\qquad\qquad+\underbrace{\sum\limits_{t=1}^T \E\Big[\mathds{1}\{\mcA(t)\cap\overline{\mcC(t)}\}\times\Delta(\mcS(t),\bmu)\Big]}_{A_3}\ ,
\end{align}
where~\eqref{eq:regret_1} is a result of Lemma~\ref{lemma:optimal_quantized}. Next, we find an upper bound for each of the terms $A_1$, $A_2$ and $A_3$ to recover the regret bound in Theorem~\ref{theorem: regret upper bound}. 

\paragraph{Upper-bounding $A_1$:}

First, we leverage~\cite[Lemma 1]{wang2022thompson} to find an upper bound on $A_1$, which we state below for completeness.
\begin{lemma}[~\cite{wang2022thompson}]
\label{lemma:A1}
    In Algorithm~\ref{alg:CMAB-GT}, we have
    \begin{align}
        \E\left[\sum\limits_{t=1}^T \mathds{1}\{i\in\mcS(t),\;\;|\bar\mu_i(t) - \mu_i|>\varepsilon\}\right]\;\leq\;1 + \frac{1}{\varepsilon^2}\ .
    \end{align}
\end{lemma}
Leveraging Lemma~\ref{lemma:A1}, it can be readily verified that the regret due to $A_1$ can be upper bounded as
\begin{align}
    A_1\;\leq\; \left (\frac{mK^2}{\varepsilon^2} + m\right )\Delta_{\max}(\bmu)\ .
\end{align}

\paragraph{Upper-bounding $A_2$:}

Next, we provide an upper-bound for the term $A_2$. First, note that under the event $\overline{\mcB(t)}\cap\mcC(t)$, the event
\begin{align}
    \mcG(t)\;:=\; \left\{ ||\btheta_{\mcS(t)}(t) - \bar\bmu_{\mcS(t)}||_1 > \frac{\Delta(\mcS(t),\bmu)}{B} - \frac{\Delta_{\min}(\bmu)}{2B} - (K^2 + 1)\varepsilon\right\}
\end{align}
holds. Furthermore, let us define the event 
\begin{align}
    \mcH(t)\;:=\; \left\{\sum\limits_{i\in\mcS(t)} \frac{1}{T_i(t)} \leq \frac{2\Big ( \frac{\Delta(\mcS(t),\bmu)}{B} - \frac{\Delta_{\min}(\bmu)}{2B} - (K^2 + 2)\varepsilon )^2}{\log(2^m|\mcI|T)}\right\}\ .
\end{align}
Subsequently, we may expand the event $\mcG(t)$ as
\begin{align}
    \mcG(t)\;=\; \mcG(t)\cap\mcH(t)\;\;\cup\;\;\mcG(t)\cap\overline{\mcH(t)}\ .
\end{align}
Next, note that using~\cite[Lemma 2]{perrault2021statistical}, it can be readily verified that 
\begin{align}
    \P\left ( \mcG(t)\cap\mcH(t)\right )\;\leq\;\frac{1}{T}\quad\forall t\in\N\ .
\end{align}
Hence, what remains is to upper-bound the term
\begin{align}
    \sum\limits_{t=1}^T \E\left[\mathds{1}\{\mcG(t)\cap\overline{\mcH(t)}\}\times\Delta(\mcS(t),\bmu)\right]\ .
\end{align}
Similarly to~\cite{wang2022thompson}, under the event $\overline{\mcH(t)}$, we define a function that upper-bounds the regret at time $t$ due to the super-arm $\mcS(t)$. Specifically, for any arm $i\in\mcS(t)$, let $g_i(T_i(t))$ denote this function, and we show that $\sum_{i\in\mcS(t)}g_i(T_i(t))\geq\Delta(\mcS(t),\bmu)$. Finally, we have
\begin{align}
    \sum\limits_{t=1}^T \E\left[\mathds{1}\{\mcG(t)\cap\overline{\mcH(t)}\}\times\Delta(\mcS(t),\bmu)\right]\;\leq\; \sum\limits_{t\in\N}\sum\limits_{i\in\mcS(t)} g_i(T_i(t))\ .
\end{align}
For any $i\in[m]$, let us define the function
\begin{align}
\label{eq: g_i}
      &g_i(n):= \left\{
	\begin{array}{ll}
	\displaystyle \Delta_{\max}(\bmu), & \mbox{if} \;\; n=0\\
    &\\
    {\displaystyle 2B\sqrt{\frac{\log(2^m|\mcI|T)}{n}}} , & {\displaystyle\mbox{if}\;\;1\leq n\leq L_{i,1}}\\
    &\\
	{\displaystyle  \frac{2B\log(2^m|\mcI|T)}{n\min\limits_{\mcS : i\in\mcS}\Big( \frac{\Delta(\mcS,\bmu)}{B} - \frac{\Delta_{\min}(\bmu)}{2} - (K^2+2)\varepsilon\Big)}} , & {\displaystyle\mbox{if}\;\; L_{i,1}<n\leq L_{i,2}} \\
    &\\
    0 , & \mbox{if}\;\; n > L_{i,2} \\
	\end{array}\right. \ ,
\end{align}
where we have defined 
\begin{align}
    L_{i,1}\;:=\; \frac{\log(2^m|\mcI|T)}{\min\limits_{\mcS : i\in\mcS}\Big ( \frac{\Delta(\mcS,\bmu)}{B} - \frac{\Delta_{\min}(\bmu)}{2B} - (K^2+2)\varepsilon\Big)^2}\ ,
\end{align}
and,
\begin{align}
    L_{i,2}\;:=\; \frac{K\log(2^m|\mcI|T)}{\min\limits_{\mcS : i\in\mcS}\Big ( \frac{\Delta(\mcS,\bmu)}{B} - \frac{\Delta_{\min}(\bmu)}{2B} - (K^2+2)\varepsilon\Big)^2}\ .
\end{align}

Next, we verify that the function $g_i$ defined in~\eqref{eq: g_i} satisfies the condition that $\sum_{i\in\mcS(t)} g_i(T_i(t))\geq \Delta(\mcS(t),\bmu)$ for every $t\in\N$. First, note that if there exists arms $j\in\mcS(t)$ such that $T_j(t)=0$, we have
\begin{align}
    \sum\limits_{i\in\mcS(t)}g_i(T_i(t))\;&\geq\; g_j(T_j(t))\\
    &=\; \Delta_{\max}(\bmu)\\
    &\geq\;\Delta(\mcS(t),\bmu)\ .
\end{align}
Next, if there exists an arm $j\in\mcS(t)$ such that 
\begin{align}
1\leq T_j(t)\leq \frac{\log(2^m|\mcI|T)}{\big ( \frac{\Delta(\mcS(t),\bmu)}{B} - \frac{\Delta_{\min}(\bmu)}{2B} - (K^2+2)\varepsilon\big)^2}\ ,    
\end{align}
which implies that $1\leq T_j(t)\leq L_{i,1}$, we have
\begin{align}
    \sum\limits_{i\in\mcS(t)} g_i(T_i(t))\;&\geq\; g_j(T_j(t))\\
    &=\; 2\sqrt{\frac{\log(2^m|\mcI|T)}{T_j(t)}}\\
    &\geq\; 2\left ( \frac{\Delta(\mcS(t),\bmu)}{B} - \frac{\Delta_{\min}(\bmu)}{2B} - (K^2+2)\varepsilon\right )\\
    &\geq\;\Delta(\mcS(t),\bmu)\ ,
    \label{eq: A2_1}
\end{align}
where~\eqref{eq: A2_1} holds since we have chosen $\varepsilon\in\R_+$ such that
\begin{align}
    \varepsilon\;<\;\frac{\Delta_{\min}(\bmu)}{4B(K^2 + 2)}\ .
\end{align}
Next, if, for all $i\in\mcS(t)$ we have 
\begin{align}
T_i(t) > \frac{\log(2^m|\mcI|T)}{\big ( \frac{\Delta(\mcS(t),\bmu)}{B} - \frac{\Delta_{\min}(\bmu)}{2B} - (K^2+2)\varepsilon\big)^2}\ ,    
\end{align}
we can decompose $\mcS(t)$ into three disjoint subsets. Specifically, we define 
\begin{align}
\mcS_1(t) &:=\{i\in\mcS(t) : T_i(t) \leq L_{i,1}\}\ , \\
\mcS_2(t) &:= \{i\in\mcS(t) : L_{i,1}<T_i(t)\leq L_{i,2}\}\ ,\\
\mcS_3(t) &: = \{i\in\mcS(t): T_i(t) > L_{i,2} \}\ .
\end{align}
Subsequently, we have
\begin{align}
    &\sum\limits_{i\in\mcS(t)} g_i(T_i(t))\nonumber\\
    &\;\;= \sum\limits_{i\in\mcS_1(t)} g_i(T_i(t)) + \sum\limits_{i\in\mcS_2(t)} g_i(T_i(t))\\
    &\;\; = \sum\limits_{i\in\mcS_1(t)} 2B\sqrt{\frac{\log(2^m|\mcI|T)}{T_i(t)}} + \sum\limits_{i\in\mcS_2(t)} \frac{2B\log(2^m|\mcI|T)}{T_i(t)\min\limits_{\mcS : i\in\mcS} \big ( \frac{\Delta(\mcS,\bmu)}{B} - \frac{\Delta_{\min}(\bmu)}{2B} - (K^2+2)\varepsilon\big )}\\
    &\;\; \geq \sum\limits_{i\in\mcS_1(t)} 2B\sqrt{\frac{\log(2^m|\mcI|T)}{T_i(t)}} + \sum\limits_{i\in\mcS_2(t)} \frac{2B\log(2^m|\mcI|T)}{T_i(t)\big ( \frac{\Delta(\mcS(t),\bmu)}{B} - \frac{\Delta_{\min}(\bmu)}{2B} - (K^2+2)\varepsilon\big )}\\
    &\;\; = \sum\limits_{i\in\mcS_1(t)} {\frac{2B\log(2^m|\mcI|T)}{T_i(t)\big ( \frac{\Delta(\mcS(t),\bmu)}{B} - \frac{\Delta_{\min}(\bmu)}{2B} - (K^2+2)\varepsilon\big )}}\times \sqrt{\frac{T_i(t)\big ( \frac{\Delta(\mcS(t),\bmu)}{B} - \frac{\Delta_{\min}(\bmu)}{2B} - (K^2+2)\varepsilon\big )^2}{\log(2^m|\mcI|T)}}\nonumber\\
    &\qquad + \sum\limits_{i\in\mcS_2(t)} \frac{2B\log(2^m|\mcI|T)}{T_i(t)\big ( \frac{\Delta(\mcS(t),\bmu)}{B} - \frac{\Delta_{\min}(\bmu)}{2B} - (K^2+2)\varepsilon\big )}\\
    \label{eq: A2_2}
    &\;\;\geq \sum\limits_{i\in\mcS_1(t)} {\frac{2B\log(2^m|\mcI|T)}{T_i(t)\big ( \frac{\Delta(\mcS(t),\bmu)}{B} - \frac{\Delta_{\min}(\bmu)}{2B} - (K^2+2)\varepsilon\big )}}+ \sum\limits_{i\in\mcS_2(t)} \frac{2B\log(2^m|\mcI|T)}{T_i(t) \big ( \frac{\Delta(\mcS(t),\bmu)}{B} - \frac{\Delta_{\min}(\bmu)}{2B} - (K^2+2)\varepsilon\big )}\\
    &\;\;= \frac{2B\log(2^m|\mcI|T)}{\big (\frac{\Delta(\mcS(t),\bmu)}{B} - \frac{\Delta_{\min}(\bmu)}{2B} - (K^2+2)\varepsilon  \big )}\times\left (\sum\limits_{i\in\mcS(t)} \frac{1}{T_i(t)} - \sum\limits_{i\in\mcS_3(t)}\frac{1}{T_i(t)}\right )\\
    \label{eq: A2_3}
    &\;\;\geq \frac{2B\log(2^m|\mcI|T)}{\big (\frac{\Delta(\mcS(t),\bmu)}{B} - \frac{\Delta_{\min}(\bmu)}{2B} - (K^2+2)\varepsilon\big )}\nonumber\\
    &\qquad\qquad\times\left ( \frac{2\big (\frac{\Delta(\mcS(t),\bmu)}{B} - \frac{\Delta_{\min}(\bmu)}{2B} - (K^2+2)\varepsilon\big )^2}{\log(2^m|\mcI|T)}- \frac{\big (\frac{\Delta(\mcS(t),\bmu)}{B} - \frac{\Delta_{\min}(\bmu)}{2B} - (K^2+2)\varepsilon\big )^2}{\log(2^m|\mcI|T)}\right )\\
    \label{eq: A2_4}
    \\&\;\;\geq \Delta(\mcS(t),\bmu)\ ,
\end{align}
where~\eqref{eq: A2_2} uses the fact that 
\begin{align}
T_i(t) > \frac{\log(2^m|\mcI|T)}{\big ( \frac{\Delta(\mcS(t),\bmu)}{B} - \frac{\Delta_{\min}(\bmu)}{2B} - (K^2+2)\varepsilon\big)^2}\ , \quad  \forall  \;\; i\in\mcS(t)\ ,    
\end{align}
and \eqref{eq: A2_3} holds due to the event $\overline{\mcH(t)}$ along with the definition of the set $\mcS_3(t)$, and~\eqref{eq: A2_4} holds by the choice of $\varepsilon\in\R_+$. Finally, if all $i\in\mcS(t)$ satisfy $T_i(t) > L_{i,2}$, we have
\begin{align}
    \sum\limits_{i\in\mcS(t)}\frac{1}{T_i(t)}\;&\leq\; \sum\limits_{i\in\mcS(t)}\frac{1}{L_{i,2}}\\
    &=\; \sum\limits_{i\in\mcS(t)} \frac{\min\limits_{\mcS : i\in\mcS}\big ( \frac{\Delta(\mcS,\bmu)}{B} - \frac{\Delta_{\min}(\bmu)}{2B} - (K^2+2)\varepsilon\big )^2}{K\log(2^m|\mcI|T)}\\
    &\leq\frac{\big ( \frac{\Delta(\mcS(t),\bmu)}{B} - \frac{\Delta_{\min}(\bmu)}{2B} - (K^2+2)\varepsilon)^2}{\log(2^m|\mcI|T)}\ ,
\end{align}
which is in contradiction with the event $\overline{\mcH(t)}$. Hence, we have shown that under the event $\overline{\mcH(t)}$, the functions $g_i$ satisfy the inequality $\sum\limits_{i\in\mcS(t)} g_i(T_i(t))\geq \Delta(\mcS(t),\bmu)$. Finally, summing up the $g_i(T_i(t))$ functions over time and the set of arms, following a similar procedure to~\citep{wang2022thompson}, we obtain that 
\begin{align}
    A_2\;\leq\; 2m\Delta_{\max}(\bmu) + \sum\limits_{i\in[m]} \left ( 2\log K + 6\right )\frac{B\log(2^m|\mcI|T)}{\min\limits_{\mcS : i\in\mcS}\big (\frac{\Delta(\mcS,\bmu)}{B} - \frac{\Delta_{\min}(\bmu)}{2B} - (K^2+2)\varepsilon\big)}\ .
\end{align}
\paragraph{Upper-bounding $A_3$:}

Finally, we turn our attention to upper-bounding $A_3$. Before analyzing the upper bound, let us lay down a few notations and definitions required in the analysis. Let $\btheta,\;\bar{\btheta}\in[0,1]^m$ and $\mcZ\subseteq[m]$. Accordingly, let us define $\btheta^\prime:=(\bar{\btheta}_{\mcZ}, \btheta_{\bar{\mcZ}})$ as a vector, whose $i^{\sf th}$ coordinate has the same value as the $i^{\sf th}$ coordinate of $\bar{\btheta}$ if $i\in\mcZ$, and otherwise, it has the same value as the $i^{\sf th}$ coordinate of $\btheta$. Let $\mcS^\star_Q(\bmu)\in\mcT_Q(\bmu)$ denote one of the optimal super-arms with respect to the quantized reward function. Furthermore, for any choice of $\bar\btheta$ and $\mcZ$ such that $||\bar\btheta_{\mcZ} - \mu_{\mcZ}||_{\infty} \leq \varepsilon$, let us consider the following properties of the vector $\btheta^\prime$.
\begin{itemize}
    \item[P1.] $\mcZ\;\subseteq\;\mcS_Q^\star(\btheta^\prime)$
    \item[P2.] Either $\mcS_Q^\star(\btheta^\prime)\in\mcT_Q(\bmu)$, or $||\btheta^\prime_{\mcS_Q^\star(\btheta^\prime)} - \bmu_{\mcS_Q^\star(\btheta^\prime)}||_1 > \frac{1}{B}\Delta(\mcS_Q^\star(\btheta^\prime),\bmu) - \frac{1}{2B}\Delta_{\min}(\bmu) - (K^2+1)\varepsilon$,
\end{itemize}
Furthermore, for any $\mcZ\subseteq[m]$ and $\btheta,\bar\btheta\in[0,1]^m$ satisfying $||\bar\btheta_{\mcZ} - \bmu_{\mcZ}||_{\infty}\leq \varepsilon$, let us define the event
\begin{align}
    \mcE_{\mcZ,1}(\btheta)\;:=\; \Big \{ \text{properties P1 and P2 hold for }\mcZ\subseteq[m]\text{ and }\btheta\in[0,1]^m\Big \}\ .
\end{align}
Additionally, let us define the event
\begin{align}
    \mcM(t)\;:=\; \Big \{ \mcS(t) \neq \mcS^\star(\btheta(t))\Big \}\  .
\end{align}
For upper-bounding $A_3$, we decompose the event $\mcA(t)\cap\bar\mcC(t)$ as follows.
\begin{align}
    \mcA(t)\cap\overline{\mcC(t)}\;=\; \mcA(t)\cap\overline{\mcC(t)}\cap\mcM(t)\;\;\cup\;\;\mcA(t)\cap\overline{\mcC(t)}\cap\overline{mcM(t)}\ .
\end{align}
Leveraging Lemma~\ref{lemma: number of tests}, we have that at any time $t\in\N$, $\P(\mcM(t))\leq\frac{1}{t^2}$. Hence, we have
\begin{align}
\label{eq: A3_T1}
    \sum\limits_{t=1}^T \E\left [ \mathds{1}\{\mcM(t)\}\times\Delta(\mcS(t),\bmu)\right]\; &<\; \Delta_{\max}(\bmu)\sum\limits_{t=1}^\infty \P(\mcM(t))\\
    &=\; \frac{\pi^2}{6}\Delta_{\max}(\bmu)\ .
\end{align}
Next, we upper-bound the regret due to $A_3$ under the event $\overline{\mcM(t)}$, i.e., when the GTO returns the same super-arm as the exact oracle. We emphasize that under the event $\mcA(t)\cap\overline{\mcC(t)}\cap\overline{\mcM(t)}$, the analysis does not reduce to the analysis of the CTS algorithm~\cite{wang2022thompson}, since, an exact oracle operates on the {\em true} reward function $r(\cdot\;;\;\cdot)$, whereas the GTO operates on the quantized reward function $Q(r(\cdot\;;\;\cdot))$. Next, we prove that if the event $\mcA(t)\cap\overline{\mcC(t)}\cap\overline{\mcM(t)}$ occurs, then it implies that there exists a set $\mcZ\subseteq\mcS^\star_Q(\bmu)$ such that the event $\mcE_{\mcZ,1}(\btheta(t))$ occurs. Before formally proving this statement, let us understand its implication. If there exists $\mcZ\subseteq\mcS^\star_Q(\bmu)$, $\mcZ\neq\emptyset$, such that $\mcE_{\mcZ,1}(\btheta(t))$ occurs, then it immediately implies that $||\btheta_{\mcZ}(t) - \bmu_{\mcZ} ||_{\infty} > \varepsilon$. This is because, if $||\btheta_{\mcZ}(t) - \bmu_{\mcZ} ||_{\infty} \leq \varepsilon$, then $\btheta(t)$ becomes a candidate choice for $\btheta^\prime$, and thus, either 1) $\mcS(t)\in\mcT_Q(\bmu)$, (hence contradicting the event $\mcA(t)$) or 2) $||\btheta_{\mcS(t)}(t) - \bmu_{\mcS(t)} ||_1 > \frac{1}{B}\Delta(\mcS(t),\bmu) - \frac{1}{2B}\Delta_{\min}(\bmu) - (K^2+1)\varepsilon$ (hence, contradicting the event $\overline{\mcC(t)}$). Subsequently, we can leverage ~\cite[Lemma 3]{wang2022thompson} which provides an upper bound on the number of times that the event $\mcE_{\mcZ,2}(t) := \{||\btheta_{\mcZ}(t) - \bmu_[\mcZ] ||_{\infty}>\varepsilon\}$ occurs.
\begin{lemma}[\cite{wang2022thompson}]
    We have
    \begin{align}
    \label{eq: A3_T2}
        \sum\limits_{t=1}^T \E\left[\mathds{1}\{\mcA(t)\cap\overline{\mcC(t)}\cap\overline{\mcM(t)}\cap\mcE_{\mcZ,2}(t)\}\right]\;\leq\;13\alpha\frac{8}{\varepsilon^2}\left ( \frac{4}{\varepsilon^2} + 1\right )^{K} \log\frac{K}{\varepsilon^2}\ ,
    \end{align}
    where $\alpha\in\R_+$ is a universal constant.
\end{lemma}
Hence, we obtain that the regret due to $A_3$ is upper-bounded by
\begin{align}
    A_3\;\leq\; \left ( 13\alpha\frac{8}{\varepsilon^2}\left ( \frac{4}{\varepsilon^2} + 1\right )^{K} \log\frac{K}{\varepsilon^2} + \frac{\pi^2}{6}\right )\Delta_{\max}(\bmu)\ .
\end{align}

What remains is to prove the following lemma.
\begin{lemma}
    If the event $\overline{\mcC(t)}\cap\mcA(t)\cap\overline{\mcM(t)}$ happens, then there exists a subset $\mcZ\subseteq \mcS^\star_Q(\bmu)$, $\mcZ\neq\emptyset$, such that $\mcE_{\mcZ,1}(\btheta(t))$ holds.
\end{lemma}
\begin{proof}
    First, let us set $\mcZ=\mcS^\star_Q(\bmu)$. Accordingly, we define the vector $\btheta^\prime$ such that $||\btheta^\prime_{\mcS^\star_Q(\bmu)} - \bmu_{\mcS^\star_Q(\bmu)} ||_{\infty}\leq \varepsilon$. We will show that for any $\mcS^\prime$ such that $\mcS^\prime\cap\mcS^\star_Q(\bmu)=\emptyset$, $\mcS^\star_Q(\btheta^\prime)\neq \mcS^\prime$. To verify this, note that
    \begin{align}
        Q(r(\mcS^\prime\;;\;\btheta^\prime))\;&=\; Q(r(\mcS^\prime\;;\;\btheta(t)))\\
        &\leq\; Q(r(\mcS(t)\;;\;\btheta(t)))\\
        &\stackrel{\eqref{eq: quantizer}}{\leq}\; r(\mcS(t)\;;\;\btheta(t)) + \frac{\Delta_{\min}(\bmu)}{4B}\\
        &\stackrel{\overline{\mcC(t)}}{\leq}\; r(\mcS(t)\;;\;\bmu) + \Delta(\mcS(t),\bmu) - B(K^2+1)\varepsilon - \frac{\Delta_{\min}(\bmu)}{4B}\\
        \label{eq:ez1_0}
        &<\;r(\mcS^\star_Q(\bmu)\;;\;\bmu) - BK\varepsilon - \frac{\Delta_{\min}(\bmu)}{4B}\\
        \label{eq:ez1_1}
        &\leq\; r(\mcS^\star_Q(\bmu)\;;\;\btheta^\prime) + BK\varepsilon - BK\varepsilon - \frac{\Delta_{\min}(\bmu)}{4B}\\
        \label{eq:ez1_2}
        &\stackrel{\eqref{eq: quantizer}}{\leq}\; Q(r(\mcS^\star_Q(\bmu)\;;\;\btheta^\prime))\ ,
    \end{align}
    where~\eqref{eq:ez1_0} is a consequence of Lemma~\ref{lemma:optimal_quantized} and~\eqref{eq:ez1_1} follows from Assumption~\ref{assumption:Lipschitz}. Hence, from~\eqref{eq:ez1_2} we conclude that $\mcS^\prime\neq\mcS^\star_Q(\btheta^\prime)$. So, we have two possibilities for $\mcS^\star_Q(\btheta^\prime)$.
    \begin{itemize}
        \item[a)] $\mcS^\star_Q(\bmu)\subseteq\mcS^\star_Q(\btheta^\prime)$.
        \item[b)] Let us define $\mcZ_1 := \mcS^\star_Q(\btheta^\prime)\cap\mcS^\star_Q(\bmu)$. Then, we have $\mcZ_1\neq\emptyset$.
    \end{itemize}
    For the case (a), if $\mcS^\star_Q(\btheta^\prime)\notin \mcT(\bmu)$, we have
    \begin{align}
        r(\mcS^\star_Q(\btheta^\prime)\;;\;\btheta^\prime)\;&>\; r(\mcS^\star_Q(\btheta^\prime)\;;\;\bmu) - BK\varepsilon\\
        &\geq r(\mcS^\star_Q(\bmu)\;;\;\bmu) + \Delta(\mcS^\star_Q(\btheta^\prime),\bmu) - BK\varepsilon\ ,
    \end{align}
    which, along with Assumption~\ref{assumption:Lipschitz}, implies that
    \begin{align}
        ||\btheta^\prime_{\mcS^\star_Q(\btheta^\prime)} -\bmu_{\mcS^\star_Q(\btheta^\prime)}||_1\;>\;\frac{\Delta(\mcS^\star_Q(\btheta^\prime),\bmu)}{B} - K\varepsilon\ ,
    \end{align}
    which implies that $\mcE_{\mcZ,1}(\btheta(t))$ holds with $\mcZ = \mcS^\star_Q(\bmu)$. Otherwise, for case (b), we follow the same set of arguments as in~\citep[Lemma 2]{wang2022thompson}, which concludes the proof.
\end{proof}
Finally, Theorem~\ref{theorem: regret upper bound} is obtained by adding the upper-bounds obtained due to the terms $A_1$, $A_2$ and $A_3$.

\section{Artificial Neural Network (ANN)}
\label{proof: ANN}

In this section, we show that a $2$-layer ANN with sigmoid activation satisfies the separability condition in Assumption~\eqref{assumption: C separability}, with some conditions on the weights. Specifically, consider a $2$-layer ANN with the hidden layer weight matrix denoted by $\bW_1$ and the output weights denoted by the vector $\bw_2$, i.e., for any input $\btheta_{\mcS}\in[0,1]^m$, the output of the neural network is given by
\begin{align}
    r(\mcS\;;\;\btheta)\;:=\;\left\langle \bw_2\;,\;\sigma\Big( \bW_1\btheta_{\mcS}\Big)\right\rangle\ ,
\end{align}
where $\sigma(x):=\frac{1}{1+e^{-x}}$ denotes the sigmoid activation function. The result is formally defined next.

\begin{theorem}
\label{theorem: ANN}
    Any $2$-layer ANN with the hidden layer $\bW_1$ and output weights $\bw_2$ is separable, i.e., for any $s\in\mcS^\star(\btheta)$ and $\tilde{s}\notin\mcS^\star(\btheta)$, and for any $\mcS\subseteq[m]\setminus\{s,\tilde{s}\}$, we have
    \begin{align}
        r(\mcS\cup\{s\}\;;\;\btheta) - r(\mcS\cup\{\tilde{s}\}\;;\;\btheta)\;>\;0\ ,
    \end{align}
    for any $\btheta\in[0,1]^m$.
\end{theorem}
\begin{proof}
    Let us denote the number of neurons in the hidden layer by $N$. The difference in rewards for any $\btheta\in[0,1]^m$ and sets $\mcS,\mcS^\prime\subseteq[0,1]^m\setminus\{s,\tilde{s}\}$ can be expanded as
    \begin{align}
        &r(\mcS\cup\{s\}\;;\;\btheta) - r(\mcS\cup\{\tilde{s}\}\;;\;\btheta)\nonumber\\
        &\;\;=\;\left\langle \bw_2\;,\;\underbrace{\sigma \Big( \bW_1,\btheta_{\mcS\cup\{s\}}\Big ) - \sigma \Big( \bW_1,\btheta_{\mcS\cup\{\tilde{s}\}}\Big )}_{:=\by}\right\rangle\\
        &\;\;=\;\sum\limits_{n=1}^N w_{2,n} y_n\ ,
        \label{eq: diff_reward}
    \end{align}
    where $w_{2,n}$ and $y_n$ denote the $n^{\sf th}$ coordinates of the vector $\bw_2$ and $\by$ for any $n\in[N]$. Furthermore, for any set $\mcS\subseteq[m]$, the $n^{\sf th}$ coordinate of the vector $\bv:=\sigma(\bW_1\btheta_{\mcS})$ is given by
    \begin{align}
        v_n\;=\; \frac{\exp\Big( -\sum\limits_{i\in\mcS}\theta_i\bW_{1,n,i}\Big )}{1+ \exp\Big( -\sum\limits_{i\in\mcS}\theta_i\bW_{1,n,i}\Big )}\ .
    \end{align}
    Accordingly, we have for any $n\in[N]$, 
    \begin{align}
        y_n\;=\; \frac{\left ( \exp\Big( -\sum\limits_{i\in\mcS}\theta_i\bW_{1,n,i}\Big )\right )\Big (\exp(-\theta_{\tilde{s}}\bW_{1.n,\tilde{s}}-\exp(-\theta_{s}\bW_{1,n,s})) \Big )}{\left (1+\exp\Big( -\sum\limits_{i\in\mcS\cup\{s\}}\theta_i\bW_{1,n,i}\Big )\right )\left (1+\exp\Big( -\sum\limits_{i\in\mcS\cup\{\tilde{s}\}}\theta_i\bW_{1,n,i}\Big )\right )}\ .
    \end{align}
    Next, let us define the quantities
    \begin{align}
    \label{eq:def_delta}
        \delta(s,\tilde{s},n)\;:=\;\Big (\exp(-\theta_{\tilde{s}}\bW_{1.n,\tilde{s}}-\exp(-\theta_{s}\bW_{1,n,s})) \Big )\ ,
    \end{align}
    and for any set $\mcS\subseteq[m]\setminus\{s,\tilde{s}\}$,
    \begin{align}
    \label{eq:def_alpha}
        \alpha(\mcS\;;\;n)\;:=\; \frac{\left ( \exp\Big( -\sum\limits_{i\in\mcS}\theta_i\bW_{1,n,i}\Big )\right )}{\left (1+\exp\Big( -\sum\limits_{i\in\mcS\cup\{s\}}\theta_i\bW_{1,n,i}\Big )\right )\left (1+\exp\Big( -\sum\limits_{i\in\mcS\cup\{\tilde{s}\}}\theta_i\bW_{1,n,i}\Big )\right )}\ .
    \end{align}
    Leveraging~\eqref{eq:def_delta} and~\eqref{eq:def_alpha}, we can write~\eqref{eq: diff_reward} as
    \begin{align}
        r(\mcS\cup\{s\}\;;\;\btheta) - r(\mcS\cup\{\tilde{s}\}\;;\;\btheta)\;=\; \sum\limits_{n=1}^N w_{2,n}\times\alpha(\mcS\;;\;n)\times\delta(s,\tilde{s},n)\ .
    \end{align}
    Next, let us set $\mcS = \mcS^\prime :=\mcS^\star(\btheta)\setminus\{s\}$. Accordingly, we have that
    \begin{align}
        r(\mcS^\star(\btheta)\;;\;\btheta) - r(\mcS^\prime\cup\{\tilde{s}\}\;;\;\btheta) \;&=\; \sum\limits_{n=1}^N w_{2,n}\times\alpha(\mcS^\prime\;;\;n)\times\delta(s,\tilde{s},n)\\
        &\geq\;\Delta_{\min}(\btheta)\ .
        \label{eq: ANN_0}
    \end{align}
    Furthermore, for any set $\mcS\subseteq[m]\setminus\{s,\tilde{s}\}$ we have
    \begin{align}
        r(\mcS\cup\{s\}\;;\;\btheta) - r(\mcS\cup\{\tilde{s}\}\;;\;\btheta)\;&=\; \sum\limits_{n=1}^N w_{2,n}\times\alpha(\mcS\;;\;n)\times\delta(s,\tilde{s},n)\\
        &=\;\sum\limits_{n=1}^N w_{2,n}\times\delta(s,\tilde{s},n)\times \alpha(\mcS^\prime\;;\;n)\times\frac{\alpha(\mcS\;;\;n)}{\alpha(\mcS^\prime\;;\;n)}\ .
        \label{eq: ANN_1}
    \end{align}
    Defining $\beta:=\min_{n\in[N]}\min_{\mcS\subseteq[m]\setminus\{s,\tilde{s}\}}\frac{\alpha(\mcS\;;\;n)}{\alpha(\mcS^\prime\;;\;n)}$, we note that $\beta\in\R_+$, since $\alpha(\mcS\;;\;n)\in\R_+$ for every $\mcS\in[m]\setminus\{s,\tilde{s}\}$ and $n\in[N]$. Hence, \eqref{eq: ANN_1} can be lower bounded as
    \begin{align}
        r(\mcS\cup\{s\}\;;\;\btheta) - r(\mcS\cup\{\tilde{s}\}\;;\;\btheta)\;&\geq\;\beta\sum\limits_{n=1}^N w_{2,n}\times\alpha(\mcS^\prime\;;\;n)\times\delta(s,\tilde{s},n)\\
        &\stackrel{\eqref{eq: ANN_0}}{\geq}\;\beta\Delta_{\min}(\btheta)\\
        &\;>\;0\ .
    \end{align}
\end{proof}

\vfill

\end{document}